\documentclass[twoside]{article}

\usepackage[accepted]{aistats2023}

\usepackage[round]{natbib}

\usepackage{xcolor}
\usepackage{amsfonts}
\usepackage{amssymb}
\usepackage{graphicx}
\usepackage{bbm}
\usepackage{amsthm}
\usepackage{float}
\usepackage{subcaption}
\usepackage{dsfont}
\usepackage{mathtools}

\usepackage{hyperref}
\hypersetup{colorlinks=true, unicode=true, linkcolor=[rgb]{0.10,0.05,0.67}, citecolor=[rgb]{0.10,0.05,0.67}, filecolor=[rgb]{0.10,0.05,0.67}, urlcolor=[rgb]{0.10,0.05,0.67}}

\newtheorem{theorem}{Theorem}[section]
\newtheorem{corollary}{Corollary}[theorem]
\newtheorem{lemma}[theorem]{Lemma}
\newtheorem{proposition}{Proposition}

\DeclareMathAlphabet\mathbfcal{OMS}{cmsy}{b}{n}

\begin{document}

% If your paper is accepted and the title of your paper is very long,
% the style will print as headings an error message. Use the following
% command to supply a shorter title of your paper so that it can be
% used as headings.
%
\runningtitle{Influence of Task Formulation on Neural Network Features}

% If your paper is accepted and the number of authors is large, the
% style will print as headings an error message. Use the following
% command to supply a shorter version of the authors names so that
% they can be used as headings (for example, use only the surnames)
%
% \runningauthor{Stewart, Bach, Berthet, Vert}

\twocolumn[

\aistatstitle{Regression as Classification: \\
Influence of Task Formulation on Neural Network Features}

\aistatsauthor{ Lawrence Stewart \And Francis Bach \And Quentin Berthet \And  Jean-Philippe Vert  }

\aistatsaddress{ INRIA \& ENS \\PSL Research University \\ Paris  \And  INRIA \& ENS  \\ PSL Research University \\ Paris \And Google Brain \\ Paris \And Google Brain\\Paris } ]

\begin{abstract}
Neural networks can be trained to solve regression problems by using gradient-based methods to minimize the square loss. However, practitioners often prefer to reformulate regression as a classification problem, observing that training on the cross entropy loss results in better performance. By focusing on two-layer ReLU networks, which can be fully characterized by measures over their feature space, we explore how the implicit bias induced by gradient-based optimization could partly explain the above phenomenon. 

We provide theoretical evidence that the regression formulation yields a measure whose support can differ greatly from that for classification, in the case of one-dimensional data. Our proposed optimal supports correspond directly to the features learned by the input layer of the network. The different nature of these supports sheds light on possible optimization difficulties the square loss could encounter during training, and we present empirical results illustrating this phenomenon.
\end{abstract}

\section{INTRODUCTION}\label{section:intro}

Two of the most commonplace supervised learning tasks are regression and classification. The goal of the former is to predict real-valued labels for data, whilst the goal of the latter is to predict discrete labels. Regression models are conventionally trained using the squared error loss, whilst classification models are typically trained using the cross-entropy loss.

Over the past years, neural networks have notably advanced scientific capabilities for both classification and regression problems \citep{deeplearningbook}. In addition, neural networks have desirable attributes, such as their ability to learn complex non-linear functions, as well as exhibiting adaptivity to low-dimensional supports, smoothness and latent linear sub-spaces  \citep[see][]{bach2017breaking}.

Some examples of advances in classification can be found in computer vision \citep{vision1_alex,vision2_resnet,vision3_inception,vision4_efficientnet} and natural language processing \citep{nlp1_seq2seq,nlp2_alignattention,nlp3_transformer}. Similarly, neural networks have achieved the state of the art on regression problems, such as pose estimation \citep{reg1_facialcasacde,reg2_deeppose,reg3_robust,reg4_3dheadpose}. Interestingly, it can be remarked that the amount of scientific work applying neural networks to classification tasks significantly outweighs that for regression problems.

The predictive power of neural networks does not come without drawbacks. Unlike kernel methods \citep{kernelsbook, harmonic_semigroups}, to which neural networks are closely related, there are no optimization guarantees for finite neural networks, which may become stuck in local minima of the loss function. The existence of such local minima is a consequence of the non-convexity of loss functions with respect to the weights of deep neural networks that have non-linearities between layers. 

The undesirable convergence to a local minimum of a loss function typically leads to under-fitting.  Local minima are often encountered in training even when the data are generated directly from a teacher neural network \citep{spurrious}. Over-parametrization can sometimes help to alleviate this problem \citep{overparam_inductive_bias, characterizing_goodfellow}, but this is not guaranteed \citep{bad_overparam}.

A commonly seen practise within the machine learning community is the transformation of regression problems into classification problems. Instead of training a neural network using the square loss function on the original regression problem, one instead trains the model using the cross entropy loss on a new discretized classification task. Such a reformulation can often yield better performance, despite the cross entropy loss having no notion of distance between classes.

There are several synonymous names referring to the above practise:  discretizing, binning, quantizing or digitizing a regression problem.  Throughout this paper, we will refer to this practise as the \textit{binning phenomenon}. We provide some examples of literature utilizing this technique, but our list is certainly not exhaustive. 

\citet{app_imagecolor} found discretizing the \textit{``ab'' color-space} yielded better predictions for image colorization. Similarly, by binning the pixel space, \citet{app_PRNN} improved upon previous regression-based approaches \citep{regpixel1, regpixel2} for generative image modelling. Reformulation of regression as classification has also led to state-of-the-art performance in the fields of age estimation \citep{age_estimation}, pose estimation \citep{pose_binning}, and reinforcement learning \citep{rubiks,deepmindmu}. The practise is also seen outside of academic research, for example in the winning solution of the \href{https://deepsense.ai/deep-learning-right-whale-recognition-kaggle/}{NOAA Right Whale Recognition} Kaggle challenge\footnote{https://deepsense.ai/deep-learning-right-whale-recognition-kaggle/}.

% there is also the petfinder kaggle https://www.kaggle.com/competitions/petfinder-pawpularity-score/discussion/300942 but much less clear

\subsection{Contributions}
The goal of this paper is to examine how the implicit bias obtained when training neural networks with gradient-based methods could provide one possible explanation to the binning phenomenon. In order to utilize recent results on optimization \citep{globalconvergence_bach} and implicit bias  \citep{classbias_bach,regbias_boursier}, we restrict ourselves to the case of two layer neural networks with the ReLU non-linearity \citep{relu}. Our contributions are the following:

\begin{itemize}
    \item We study two simplified problems which closely relate to the implicit biases induced when training over-parameterized models on the square and cross entropy losses, in the case of one-dimensional data. In particular, we provide supports of optimal measures for both of these problems. These supports correspond directly to the features learnt by finite networks.
    \item We postulate that a sparse optimal support for the regression implicit bias could result in optimization difficulties, shedding light on one possible explanation for the binning phenomenon. We provide synthetic experiments which exhibit this behaviour.
    % \item  We postulate the (possibly) sparse support for regression could induce difficulties in optimization, which sheds light on one possible cause of the binning phenomenon. We provide synthetic experiments which exhibit  this behaviour.
\end{itemize}

The code to reproduce our experiments can be found at \url{https://github.com/LawrenceMMStewart/Regression-as-Classification}.

\subsection{Limitations}
Our analysis and empirical results only demonstrate the link between implicit biases and the binning phenomenon for two-layer neural networks. Experimentation showed that deeper models did not suffer under-fitting on our synthetic problem when trained on the square loss (see Appendix \ref{appendix:3layer}). Secondly, the optimal supports we propose are for problems that closely resemble the implicit biases of \citet{regbias_boursier,classbias_bach}. The re-parameterization we invoke to simplify analysis of the feature space introduces a factor into the total variation, which for simplicity we ignore. Finally, the link between our proposed supports and optimization is only seen empirically. Producing theory to describe whether or not a regression problem will encounter optimization difficulties as a consequence of implicit biases remains a difficult open problem.

\subsection{Notation}
 For any $n\in \mathbb{N}$, let $[n] = \{1, \ldots, n\}$. For a vector $x\in\mathbb{R}^d$ and $l\in [d]$, let $x_{[l]}\in\mathbb{R}^l$ denote the vector consisting of the first $l$ indices of $x$. Let $e_j$ denote the $j^{\small{th}}$ canonical basis vector of $\mathbb{R}^k$. Let $S^{d-1} =\{x\in\mathbb{R}^d : \: \lVert x \rVert_2 = 1 \}$. Let $\left(\cdot\right)_+ = \text{max}(\cdot,0)$ denote the ReLU non-linearity, where the maximum is taken element-wise. Let $\Omega_*$ denote the dual norm of $\Omega$, a norm on $\mathbb{R}^k$.  Let $\mathds{1}(x=v)$ denote the indicator function, taking the value of $1$ if $x=v$, otherwise $0$ for $x\not=v$.  Let $I_S:\mathbb{R}^k \rightarrow \{0, \infty\} $ denote the characteristic function of convex set $S\subseteq\mathbb{R}^k$, where $I_S(y)=0$  if $y\in S$, otherwise $I_S(y)=\infty$.  Let $\sigma_S$ denote the support function of convex set $S\subseteq\mathbb{R}^k$, defined as $\sigma_{S}(y) = \sup\limits_{w\in S}w^Ty$. Let $s:\mathbb{R}^k \rightarrow \mathbb{R}^k $ denote the softmax function, where $\left(s(v)\right)_j = {e^{v_j}}/ {\sum_{l=1}^k e^{v_l}}$.

\section{FORMULATING REGRESSION AS CLASSIFICATION}\label{section:binning}
 Let $(x_1,y_1),\ldots, (x_n,y_n) \in\mathbb{R}^d\times [0,1]$ denote the train data set for a regression problem, where we have assumed without loss of generality that the labels $y_1,\ldots,y_n$ have been normalized to the unit interval. To discretize the regression data, divide the interval $[0,1]$ into $k$ bins with midpoints given by  $\lambda \in \mathbb{R}^k$, where $0=\lambda_1<\cdots<\lambda_k=1$. The new discrete labels $\tilde{y}_i\in \arg\min_{j\in[k]} |y_i - \lambda_j|$ correspond to which of the $k$ bins each of the $y_i$ falls into, taking the left-most bin in case of ties. Figure \ref{fig:binning_diag} visually depicts this process.

The newly discretized data $\{(x_i,\tilde{y}_i)\}_{i=1}^n$ can then be used to train a classifier $f:\mathbb{R}^d \rightarrow \mathbb{R}^k$. If obtaining a real-valued prediction is imperative, one can take the expected value over the bins $s\left(f(x)\right)^T\lambda \in \mathbb{R}$.

\begin{figure}[t]
    \centering
    \includegraphics[width=0.8\linewidth]{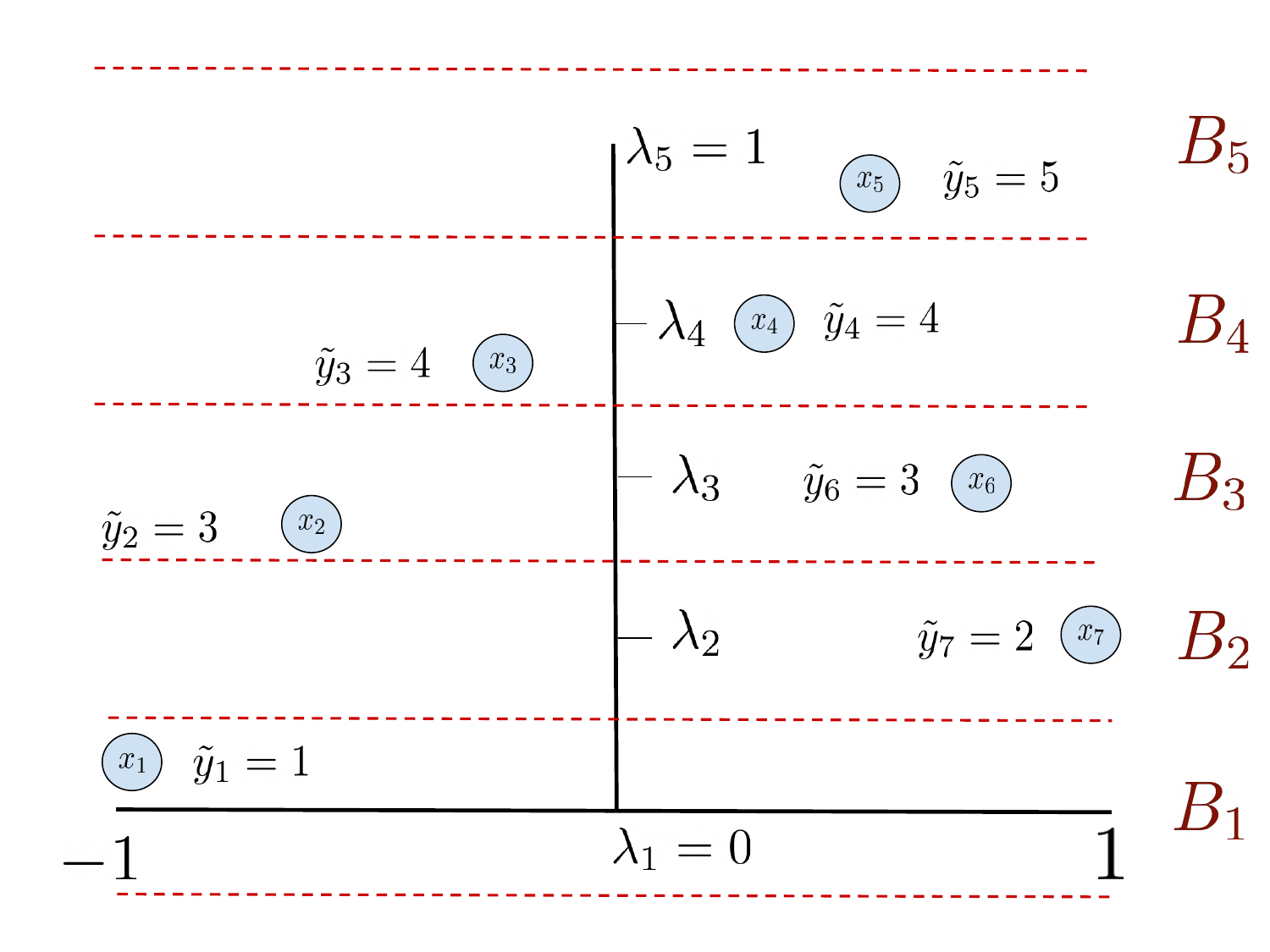}
    \caption{Depiction of binning / discretizing regression data $\{(x_i,y_i)\}_{i=1}^6$ using $k=5$ bins $B_1, \dots, B_k$, each of uniform size with midpoints $0=\lambda_1< \cdots< \lambda_k = 1$. Here $x_i\in[-1,1]$ and $y_i\in [0,1]$. The new labels $\tilde{y}_i \in [k]$ correspond to which of the $k$ bins $B_1,\ldots,B_k$ the labels $y_i$ fall into. }
    \label{fig:binning_diag}
    % \vspace{0}
\end{figure}

\section{NEURAL NETWORKS}

\subsection{Finite Sized Neural Networks}
Let $x\in \mathbb{R}^d$ be a vector whose final entry is one\footnote{This notation combines the \textit{constant terms} of neural networks with the parameters (instead of treating them separately) by appending one to the data vector.}, i.e., $x_d = 1$ and $x_{[d-1]}\in\mathbb{R}^{d-1}$. Let $a=\left( a_1, \ldots, a_m\right) \in \mathbb{R}^{m\times d}$ and $b = \left( b_1, \ldots, b_m\right) \in \mathbb{R}^{m\times k}$ denote matrices which we refer to as the input layer and output layer respectively. A two-layer ReLU neural network $F_{a,b}:\mathbb{R}^d \rightarrow \mathbb{R}^k$ is defined as:
\begin{equation}\label{eq:NeuralNetwork}
\forall x \in \mathbb{R}^d\,,\quad
    F_{a,b}(x) = \sum_{j=1}^m b_j(a_j^Tx)_+ \,.
    %\quad \in\mathbb{R}^k. \\
\end{equation}
The above equation is equivalent to the common convention of writing the linear and constant terms of the model separately:
\begin{equation}
    F_{a,b}(x) = \sum_{j=1}^m b_j \big(\,{\underbrace{a_{j,[d-1]}}_{\text{linear}}}^Tx_{[d-1]} + \underbrace{a_{j,d}}_{\text{constant}}\big)_+ .
\end{equation}
A two-layer neural network can be thought of as a model that jointly learns a set of features $\big\{(\,a_j^T\, \cdot\, )_+\big\}_{j=1}^m$ and a linear weighting $\{b_j\}_{j=1}^m$ over these features.

Since the ReLU is positively homogeneous, one can re-normalize the weights $a_j \leftarrow \frac{a_j}{\lVert a_j \rVert}$ and $b_j \leftarrow b_j \, \lVert a_j \rVert$ so that $a_j \in S^{d-1}$, without affecting $F_{a,b}$. Without loss of generality, we will assume throughout that $F_{a,b}$ has layers re-normalized in such fashion.

\subsection{Infinite Width Neural Networks}

An extension of the above is to consider models that learn a linear weighting over the set of all features $\{(a^T\cdot)_+ : a\in S^{d-1}\}$. Such models are called infinite-width neural networks and are expressed via measures, which now take the place of the output layer $b$.

Let $\mathcal{M}(S^{d-1}, \mathbb{R}^k)$ be the set of signed Radon measures \citep{rudin,measure} over $S^{d-1}$ taking values in $\mathbb{R}^k$. An infinite width network characterized by $\nu\in \mathcal{M}(S^{d-1}, \mathbb{R}^k)$ is defined as:

\begin{equation}\label{eq:MeasureNeuralNetwork}
 F_\nu(x) = \int_{S^{d-1}} (a^Tx)_+ d\nu(a) \quad \in \mathbb{R}^k.
\end{equation}

The finite models described by equation \eqref{eq:NeuralNetwork} can also be expressed in the infinite-width form by taking $\nu^{(a,b)} = \sum_{j=1}^m b_j \delta_{a_j}$. With a slight abuse of notation we can write $F_{a,b} = F_{\nu^{(a,b)}}$ to represent this.

\section{IMPLICIT BIAS}

Gradient-based optimization methods can result in a preference for certain solutions to a problem, known as an implicit bias. Possibly the simplest example of this is logistic regression (with no regularization), where training a linear predictor on a linearly separable dataset via (stochastic) gradient descent yields a solution that converges to the max-margin solution \citep[Theorem 3]{soudry2018implicit}. Similar results hold for least-squares linear regression \citep{geombias}.

The implicit bias of both linear neural networks \citep{classimpl_linear1, classimpl_linear2, classimpl_linear3} and homogeneous neural networks \citep{classimpl_homo, classbias_bach} has been studied for models trained to minimize a classification loss function with exponential tails, such as the cross entropy and exponential loss. Similar results exist for finite width two-layer networks trained with the square loss on regression problems \citep{regbias_boursier}.

\subsection{Regression}
Let $x_1, \ldots, x_n \in \mathbb{R}^d$ be data with labels $y_1, \ldots, y_n\in\mathbb{R}$. With assumptions on the data\footnote{Whilst the implicit bias for models trained on the square loss is observed empirically in experiments, the proof is restricted only to the case of orthonormal data.}, \citet[Section 3.2]{regbias_boursier} show that the gradient flow for a two-layer ReLU network trained on the square loss converges to a measure solving the following problem:
\begin{equation}\label{eq:boursier_bias} \begin{aligned} & \inf\limits_{\nu\in\mathcal{M}(S^{d-1}, \mathbb{R})} &&\int_{S^{d-1}}\left\lvert d\nu(a) \right\rvert \\ & \text{subject to}   &&F_\nu(x_i) = y_i \quad \forall i\in [n].\end{aligned} \end{equation}

For finite sized neural networks, this implicit bias selects networks which have minimum $\ell_1$-norm on their output layer from the set of all networks achieving zero square loss on the train set.

\subsection{Classification}
Let $x_1, \ldots, x_n \in \mathbb{R}^d$ be data with discrete labels $y_1, \ldots, y_n\in[k]$. Extending \citet[Theorems 3 and 5]{classbias_bach} from the logistic to soft-max loss (which corresponds to using Theorem 7 from~\citet{soudry2018implicit} instead of Theorem 3, see also Appendix \ref{appendix:bach_extension}), the gradient flow for an infinitely sized neural network trained on the cross entropy loss (multi-class classification) converges to a solution of:
  \begin{equation}\label{eq:bach_bias} \begin{aligned} & \inf\limits_{\nu\in\mathcal{M}(S^{d-1}, \mathbb{R}^k)} &&\int_{S^{d-1}} \left\lVert d\nu(u) \right\rVert  & \\ & \text{subject to}   &&(e_{y_i} - e_l)^T F_\nu(x_i) \geq  \mathds{1}(y_i\not=l), \\ & && \forall i\in [n], \quad \forall l\in[k].\end{aligned} \end{equation}

From the viewpoint of finite networks, the above implicit bias selects models whose output layer weight matrix is of minimum $\ell_1 / \ell_2$ group norm \citep[Section 1.3]{bach_l1book} from the set of all networks who satisfy a hard-margin constraint on class predictions for the train set. 

\section{RE-PARAMETERIZATION}\label{section:reparam}
\subsection{Change of Variable}

In this section, we re-parameterize the feature space $S^1$ of the infinite-width networks described in equation \eqref{eq:MeasureNeuralNetwork}, in the case of one-dimensional data. This allows us to study simplified problems that are closely related to problems \eqref{eq:boursier_bias} and~\eqref{eq:bach_bias}.

For the case of real-valued data $x\in \mathbb{R}$, we modify the notation of equation \eqref{eq:MeasureNeuralNetwork} to write:

\begin{equation}\label{eq:1dnet}
    F_\nu(x) = \int_{S^1} (a_1 x + a_2)_+ \, d\nu(a_1,a_2).
\end{equation}

Each input weight $(a_1, a_2) \in S^1$ corresponds to a feature $\psi_a(x)=  (a_1x + a_2)_+$, which is piece-wise linear with slope $a_1$ at the `active part' of the ReLU. We note that the two poles $(0, 1)$ and $(0, -1)$ correspond to the constant features $\psi_{(0,1)}(x) = 1$ and $\psi_{(0, -1)}(x) =0$. Defining $\tilde{S^1} = S^1\setminus \{(0,1), \, (0, -1)\}$, we can hence rewrite equation \eqref{eq:1dnet} as:
\begin{equation}
    F_\nu(x) = \int_{\tilde{S^1}} (a_1 x + a_2)_+ \, d\nu(a_1,a_2) \: + \: \nu\left( (0,1)\right).
\end{equation}
For the sake of simplicity, we restrict our analysis to the set of measures $\mathcal{M}(\tilde{S^1}, \mathbb{R}^k)$, which corresponds to the same set of neural networks as $\mathcal{M}(S^1, \mathbb{R}^k)$, up to a constant. We will later see through the proofs of Section \ref{section:optimal} that such a simplification is indeed permitted; for the implicit bias problems we will study, any missing constants $\nu((0,1))$ only lead to changes in the weightings of boundary features of the re-parameterized feature space.

The rough idea behind our re-parameterization is to utilise the positive-homogeneity of the ReLU to normalize the input-layer weights by the slope magnitude of their corresponding features. After re-parameterization, all features will have slopes of unit magnitude. This simplifies analysis, as the slopes of piece-wise linear segments of finite neural networks will now be controlled entirely by the network's output layer.

More formally, let $\mathbb{W}= \{-1,1\} \times \mathbb{R}$, and consider the Borel measurable function:

\begin{equation}\label{eq:measurablefunc}
\begin{array}{cccc}
   G:&\tilde{S^1}  &\longrightarrow &\mathbb{W}  \\
    &(a_1, a_2) &\longmapsto &\left(\frac{a_1}{|a_1|}, -\frac{a_2}{|a_1|}\right).
\end{array}
\end{equation}

We denote the re-parameterized input-layer weights as $(s, c) \coloneqq  G((a_1,a_2))$, and perform the following change of variable using $G$:
\begin{equation}\label{eq:changeofvar}
\begin{aligned}
    F_\nu(x) &= \int_{\tilde{S^1}} (a_1 x + a_2)_+ d\nu(a_1, a_2) \\
    &= \int_{\tilde{S^1}} \left( \frac{1}{|a_1|}(a_1x+a_2)  \right)_+ \,|a_1|d\nu(a_1, a_2) \\
    &= \int_\mathbb{W} \left(s(x-c)\right)_+ \, d\mu(s,c) \: \coloneqq  \: f_\mu(x),
    \end{aligned}
\end{equation}

where $\mu$ is the push-forward measure of $|a_1|d\nu(a_1,a_2)$ by $G$. The above change of variable defines a natural mapping:

\begin{equation}
    \begin{array}{ccc}
    T:\mathcal{M}(\tilde{S^1}, \mathbb{R}^k) &\longrightarrow&  \mathcal{M}(\mathbb{W}, \mathbb{R}^k)  \\
      \nu &\longmapsto& \mu, 
\end{array}
\end{equation}

where $\mu = T(\nu) \Longrightarrow F_\nu = f_{\mu}$. 
One can think of $c$ as the critical point or `kink' of an input weight, that is the point of discontinuity in the ReLU of the corresponding feature. $s$ can be thought of as the sign of the feature; when $s=1$ / $-1$ the feature ramps rightwards / leftwards.

\begin{figure}[t]
    \centering
    \includegraphics[width=0.80\linewidth]{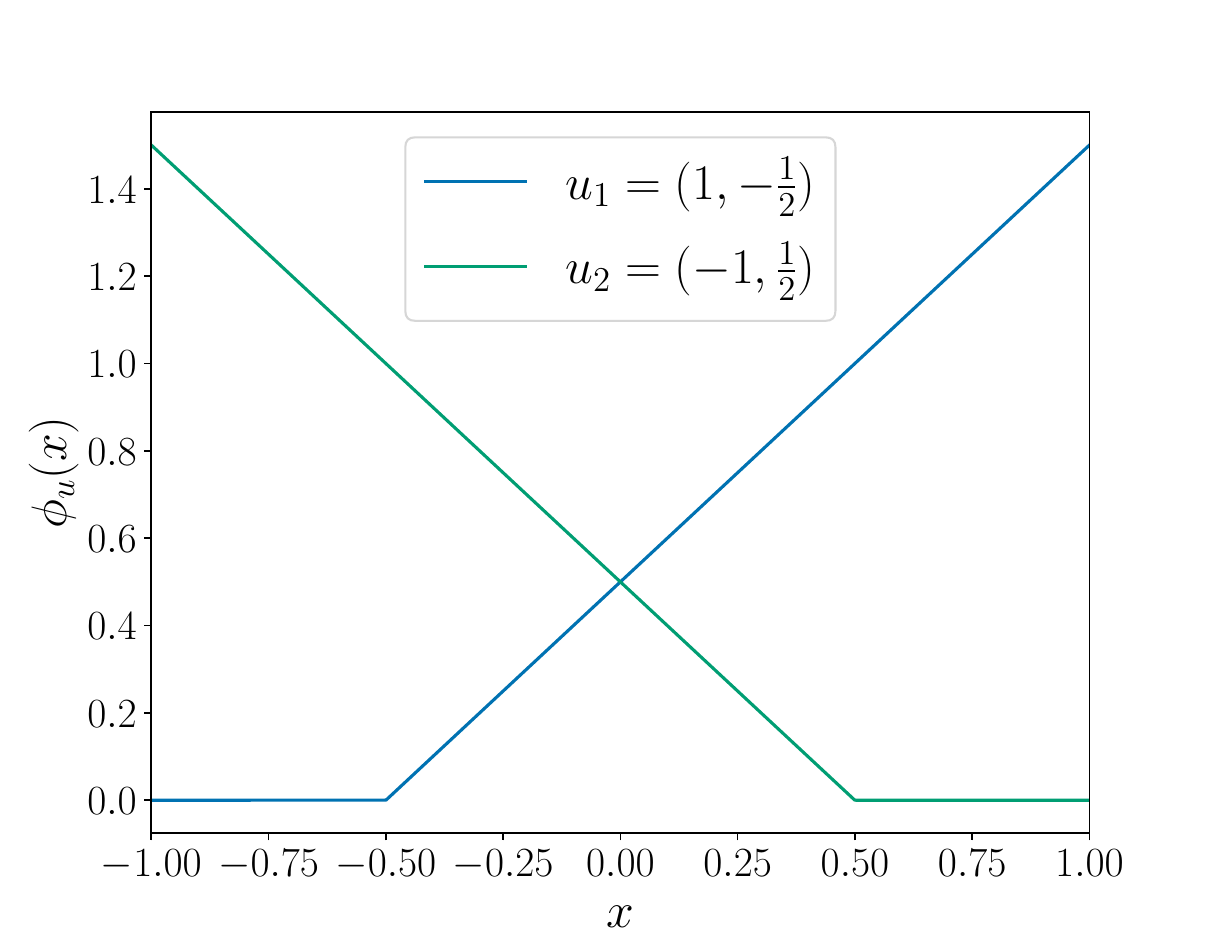}
    \caption{A depiction of two features $\phi_{u_1}$ and $\phi_{u_2}$, where $u_1=(1, -\frac{1}{2})$ is a right-ramping feature with kink at $\frac{-1}{2}$ and $u_2=(-1,\frac{1}{2})$ is a left-ramping feature with kink at $\frac{1}{2}$.}
    \label{fig:kinks}
\end{figure}

 For short-hand, we write $u=(s,c)$ and abbreviate the re-parameterized features as:
\begin{equation}\label{eq:simplefeat}
    \phi_u(x) = \left( s(x-c)\right)_+.
\end{equation}

 Figure \ref{fig:kinks} depicts an example of features in the re-parameterized space $\mathbb{W}$. Using the above notation, we can write:
 \begin{equation}\label{eq:infnn_renorm_phi}
    f_\mu( x)= \int_{\mathbb{W}} \phi_u(x) \, d\mu(u).
\end{equation}

\subsection{Simplified Implicit Biases}

Without loss of generality\footnote{Considering $\mathbb{U}$ over $\mathbb{W}$ is purely a syntactic preference in order to keep the kinks of features within the unit interval; all results and proofs generalize to $\mathbb{W}$.}, we restrict our analysis to $\mathcal{M}(\mathbb{U}, \mathbb{R}^k)$, where $\mathbb{U} = \{-1, 1\} \times [-1,1]$. Let $-1=x_1< \cdots < x_n = 1$ be ordered, real-valued data; such data can be obtained by max-min re-scaling. We define the following two problems:

\textbf{Regression:}
\begin{equation}\label{eq:implbias_reg} \begin{aligned} & \inf\limits_{\mu\in\mathcal{M}(\mathbb{U}, \mathbb{R})} &&\int_{\mathbb{U}} \:\left\lvert d\mu(u) \right\rvert \\ & \text{subject to}   &&f_\mu(x_i) = y_i \quad \forall i\in [n].\end{aligned} \end{equation}
\textbf{Classification:}
  \begin{equation}\label{eq:implbias_class} \begin{aligned} & \inf\limits_{\mu\in\mathcal{M}(\mathbb{U}, \mathbb{R}^k)} &&\int_{\mathbb{U}} \:\left\lVert d\mu(u) \right\rVert  & \\ & \text{subject to}   &&(e_{y_i} - e_l)^T f_\mu(x_i) \geq  \mathds{1}(y_i\not=l), \\ & && \forall i\in [n], \quad \forall l\in[k].\end{aligned} \end{equation}

The above problems correspond to the implicit biases of equations \eqref{eq:boursier_bias} (regression) and \eqref{eq:bach_bias} (classification). Despite the equivalence $f_{T(\nu)} =F_\nu$, problems \eqref{eq:implbias_reg} and \eqref{eq:implbias_class} are different due to the factor of $\lvert a_1 \rvert$ introduced into the total-variation when performing the change of variable. However, problems \eqref{eq:implbias_reg} and \eqref{eq:implbias_class} are easier to work with, as discussed in Section \ref{section:reparam}.

\section{OPTIMAL SUPPORTS}\label{section:optimal}

\textbf{Regression Support.}
Let $-1=x_1< \cdots < x_n = 1$ be ordered data with corresponding real labels $y_1,\ldots, y_n\in \mathbb{R}$. We define:
\begin{equation*}
    R_{reg} = \left\{x_1, x_n\right\} \cup \left\{  x_i\: : \: \frac{y_{i+1} - y_{i}}{x_{i+1}-x_{i}} \not= \frac{y_{i} - y_{i-1}}{x_{i}-x_{i-1}}   \right\} .
\end{equation*}
In words, $R_{reg}$ contains $\{x_1,x_n\}$ and any points which lie at the meeting of two line segments of the piece-wise interpolant for the data $\{(x_i,y_i)\}_{i=1}^n$. A visual example of $R_{reg}$ can be found in Figure \ref{fig:suppex}. We further define $F_{reg} = \{-1,1\} \times R_{reg}$ as the set of input weights whose features have kinks located at points appearing in $R_{reg}$.

\begin{figure}[t]
    \centering
    \includegraphics[width=0.85\linewidth]{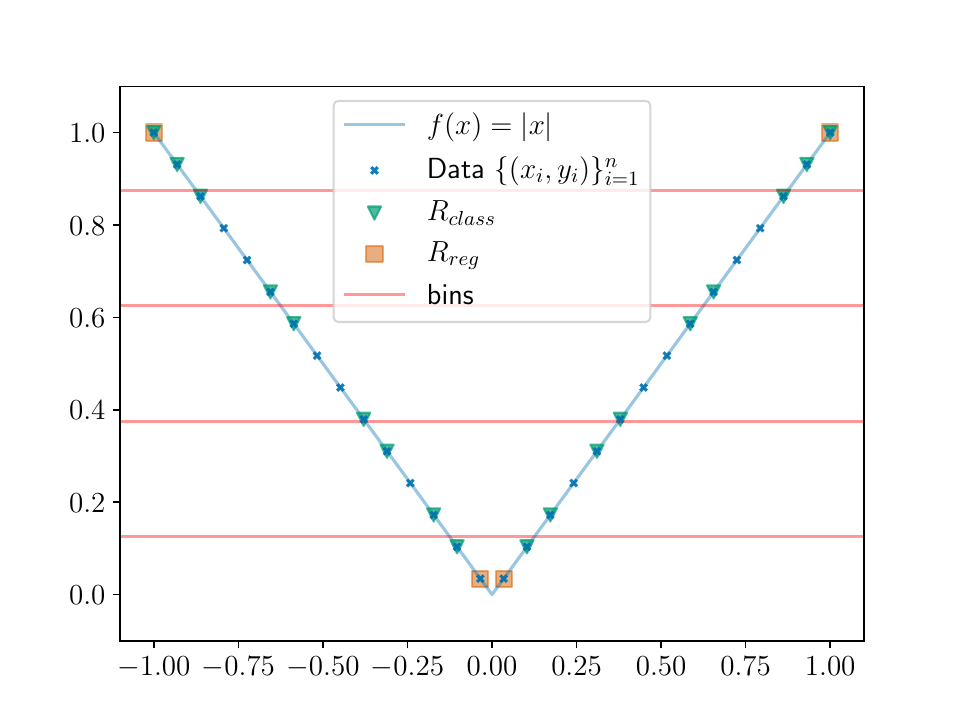}
    \caption{$R_{reg}$ and $R_{class}$ for regression data taken from the function $x\mapsto\lvert x \rvert$.}
    \label{fig:suppex}
\end{figure}

\textbf{Classification Support.}
Let $-1=x_1< \cdots < x_n = 1$ be ordered data with corresponding discrete labels $y_1,\ldots, y_n\in [k]$. We define the set $R_{class}$ as follows: \begin{equation*}
    R_{class} = \{x_1, x_n\} \cup \{ x_i : y_{i-1} \not = y_i \; \text{or} \; y_{i+1} \not= y_i\}.
\end{equation*}
In words, $R_{class}$ contains $\{x_1, x_n\}$ and all other $x_i$ which have a differing label from either of its two adjacent neighbours in the sequence $(x_i)_{i=1}^n$. An example of $R_{class}$ is depicted in Figure \ref{fig:suppex}. Similarly, we define $F_{class}=\{-1,1\}\times R_{class}$ as the set of input weights whose features have kinks located at points appearing in $R_{class}$.

We are now ready to state our main theoretical result, which shows how the implicit biases of regression \eqref{eq:implbias_reg} and classification \eqref{eq:implbias_class} differ in support.
\begin{theorem}\label{theorem:opt_supp}
For real-valued, ordered data $-1=x_1< \cdots < x_n = 1$:
\begin{enumerate}
    \item There exists $\mu \in\mathcal{M}(\mathbb{U}, \mathbb{R})$ with $\text{supp}(\mu) \subseteq F_{reg}$ which is optimal for problem \eqref{eq:implbias_reg} with labels $y_1,\ldots, y_n\in\mathbb{R}$.
    \item There exists $\nu \in\mathcal{M}(\mathbb{U}, \mathbb{R}^k)$ with $\text{supp}(\nu) \subseteq F_{class}$ which is optimal for problem \eqref{eq:implbias_class} with labels $y_1,\ldots, y_n\in[k]$.
\end{enumerate}
\end{theorem}

\textbf{Remark:} The optimal support $R_{reg}$ depends completely on the data set $\{(x_i,y_i)\}_{i=1}^n$, whilst $R_{class}$ depends both on the data and the number of bins $k$ used for discretization. In general, by increasing $k$, one can increase the size of the $R_{class}$\footnote{Excluding trivial problems, for example, when the target regression function is very close to being constant.}. This additional dependence on $k$ gives $R_{class}$ the potential to include more points than $R_{reg}$. It is not hard to think of simple regression problems for which $R_{reg}$ is sparse amongst $\{x_1,\ldots, x_n\}$, but where $R_{class}$ is not (for a suitable choice of $k$). We will explore this idea further in Section \ref{section:mst}, and its relationship to the binning phenomenon.

\subsection{An Outline for the Proof of Theorem \ref{theorem:opt_supp}}
% \textbf{Outline for the proof of Theorem \ref{theorem:opt_supp}}
\begin{enumerate}
    \item We begin in Section \ref{section:generalprob} by introducing a general optimization problem \eqref{eq:general_mintv} which encompasses both problems \eqref{eq:implbias_reg} and \eqref{eq:implbias_class}. We derive the dual of this problem in Lemma \ref{lemma:dualgenprob}.
    \item Let $U_X = \{-1,1\} \times \{x_1, \ldots, x_n\}$ be the set of features having kinks at position of the data. We aim to show there exists optimal measures for problems \eqref{eq:implbias_reg} and \eqref{eq:implbias_class}, whose supports are subsets of $U_X$. The proof of this Proposition is broken into smaller results:\begin{enumerate}
        \item In Lemma \ref{lemma:sufficient_dualfeasible} we derive a sufficient condition for dual feasibility to  problem \eqref{eq:general_mintv}.
        \item We show Corollary \ref{corollary:feasilble_implies_optimal}, which states that the existence of a feasible measure with support in $U_X$ is a sufficient condition for the existence of an optimal measure for problem \eqref{eq:general_mintv}, having support in $U_X$. 
        \item We construct feasible measures for both both problems \eqref{eq:implbias_reg} and \eqref{eq:implbias_class} in order to apply Corollary \ref{corollary:feasilble_implies_optimal} and conclude the proof of Proposition~\ref{prop:restrict_data}.
    \end{enumerate}
\item  We apply Proposition \ref{prop:restrict_data} to problems \eqref{eq:implbias_reg} and \eqref{eq:implbias_class}, but for data sets consisting only of points in $R_{reg}$ and $R_{class}$. We then extend these solutions to the complete train data set $\{x_1, \ldots, x_n\}$, and show that strong duality is indeed attained, which concludes the proof of Theorem~\ref{theorem:opt_supp}.
\end{enumerate}

\subsection{A Generalized Implicit Bias Problem}\label{section:generalprob}

Let $\Omega$ be any norm on $\mathbb{R}^k$. For a family of non-empty closed convex sets $S_1, \ldots, S_n \subseteq \mathbb{R}^{k}$ we define the following optimization problem:
  \begin{equation}\label{eq:general_mintv} \inf\limits_{\mu\in\mathcal{M}(\mathbb{U}, \mathbb{R}^k)}\int_\mathbb{U} \Omega\left( d\mu(u) \right) + \sum_{i=1}^n I_{S_i}(f_\mu(x_i)). \end{equation}

By setting $\Omega$ to be the Euclidean norm on $\mathbb{R}^k$ and choosing $k$ and $S_i$, one can recover both problems \eqref{eq:implbias_reg} and  \eqref{eq:implbias_class}. More precisely, by setting $k=1$ and $S_i=\{y_i\}$, we obtain problem \eqref{eq:implbias_reg}. On the other hand, taking $k>1$ and 
\begin{align*}
    S_i = \{ v\in \mathbb{R}^k &: (e_{y_i} - e_l)^Tv\geq \mathds{1}(y_i\not=l) \\ & \forall i\in [n], \quad \forall l\in [k] \quad  \},
\end{align*}
we recover problem \eqref{eq:implbias_class} where the data have discrete labels $y_1,\ldots, y_n \in [k]$.

\begin{lemma}\label{lemma:dualgenprob}

The dual of problem \eqref{eq:general_mintv} is:
  \begin{equation}\label{eq:general_dual} \begin{aligned} & \underset{\alpha_1\ldots,\alpha_n \in \mathbb{R}^{k}}{\sup} & &
    -\sum_{i=1}^n \sigma_{S_i}(\alpha_i)\\ & \text{subject to} & & \Omega_*\left(
    \sum_{i=1}^n \alpha_i \phi_u(x_i) \right)\leq
1 \quad \forall u\in \mathbb{U}. \end{aligned} \end{equation}
\end{lemma}
\begin{proof}
   The full proof is given in Appendix \ref{appendix:derive_dual}. We provide a brief outline. By Fenchel duality \citep{bidual}, we have:
\begin{equation*}
    \sum_{i=1}^n I_{S_i}(f_\mu(x_i)) = \sum_{i=1}^n \sup\limits_{\alpha_i\in\mathbb{R}^k} \left\{ \left\langle \alpha_i, f_\mu(x_i) \right\rangle - \sigma_{S_i}(\alpha_i) \right\}, \end{equation*}
The dual problem can be obtained by substituting this into problem \eqref{eq:general_mintv} and exchanging the order of the supremum and infinum. In order to resolve the infinum, we use properties of the dual norm.
\end{proof}

\subsubsection{Restricting the Position of Kinks to the Data}

Let $U_X= \{-1,1\} \times \{x_1, \ldots, x_n\}$ denote the set of input weights whose features have kinks at $\{x_1,\ldots, x_n\}$.

\begin{proposition}\label{prop:restrict_data}
For real-valued, ordered data $-1=x_1< \cdots < x_n = 1$:
\begin{enumerate}
    \item There exists $\mu \in\mathcal{M}(U_X, \mathbb{R})$ which is optimal for problem \eqref{eq:implbias_reg} with labels $y_1,\ldots, y_n\in\mathbb{R}$.
    \item There exists $\nu \in\mathcal{M}(U_X, \mathbb{R}^k)$ which is optimal for problem \eqref{eq:implbias_class} with labels $y_1,\ldots, y_n\in[k]$.
\end{enumerate}
\end{proposition}

To prove Proposition \ref{prop:restrict_data}, we show a series of lemmas that combine to give the desired result.

\begin{lemma}\label{lemma:sufficient_dualfeasible}
 Suppose $\alpha_1,\ldots,\alpha_n\in\mathbb{R}^k$ satisfy:
    \begin{equation*}
        \Omega_*\left(
    \sum_{i=1}^n \alpha_i \phi_u(x_i) \right)\leq 
1 \quad \forall u\in U_X.
    \end{equation*}
    Then $\alpha_1\ldots,\alpha_n$ are feasible for problem \eqref{eq:general_dual}.
\end{lemma}
\begin{proof}
The full proof is detailed in Appendix \ref{appendix:dual_Ux}, and relies on the convexity of $\Omega_*$ combined with the fact that $ \sum_{i=1}^n \alpha_i \phi_u(x_i)$ is piece-wise affine in $c$ for both left and right-wards ramping features.
\end{proof}

\begin{corollary}\label{corollary:feasilble_implies_optimal}
Suppose there exists $\mu$ feasible for the following problem: \begin{equation}\label{eq:mintv_restrictdata} \inf\limits_{\mu\in\mathcal{M}(U_X, \mathbb{R}^k) }\int_\mathbb{U} \Omega\left( d\mu(u) \right) + \sum_{i=1}^n I_{S_i}\left(f_\mu(x_i)\right).\end{equation} Then there exists $\mu^* \in \mathcal{M}(U_X, \mathbb{R}^k)$ which is optimal for problem \eqref{eq:general_mintv}.
\end{corollary}

\begin{proof}
    Let $P_1$, $D_1$ denote respectively the primal and dual values for problem \eqref{eq:general_mintv}. Similarly let $P_2$, $D_2$ denote the primal and dual values for problem \eqref{eq:mintv_restrictdata}. Since $\mathcal{M}(U_X, \mathbb{R}^k) \subset \mathcal{M}(\mathbb{U}, \mathbb{R}^k)$, it follows that $P_2 \geq P_1$ and $D_2 \geq D_1$.  
    
    Problem \eqref{eq:mintv_restrictdata} is a norm minimization problem with convex constraints which has a feasible point $\mu$, so it attains strong duality \citep[Chapter 5]{boyd2004convex}. Let $(\mu^*, \alpha^*)$  denote any primal-dual pair which attains strong duality. By Lemma \ref{lemma:sufficient_dualfeasible}, $\alpha^*$ is also dual-feasible for problem \eqref{eq:general_dual} so $D_2 = D_1$. We conclude that: \begin{equation*}
        P_2 \geq P_1 \geq D_1 = D_2 = P_2 \quad \Longrightarrow P_2 = P_1,
    \end{equation*}
    so $(\mu^*, \alpha^*)$ are optimal for problem \eqref{eq:general_mintv}.
 \end{proof}

\begin{lemma}\label{lemma:exist_feasible}
For real-valued, ordered data $-1=x_1< \cdots < x_n = 1$:
\begin{enumerate}
    \item There exists $\mu \in\mathcal{M}(U_X, \mathbb{R})$ which is feasible for problem \eqref{eq:implbias_reg} with labels $y_1,\ldots, y_n\in\mathbb{R}$.
    \item There exists $\nu \in\mathcal{M}(U_X, \mathbb{R}^k)$ which is feasible for problem \eqref{eq:implbias_class} with labels $y_1,\ldots, y_n\in[k]$.
\end{enumerate}
\end{lemma}
\begin{proof}
    The proof is constructive and detailed in Appendix~\ref{appendix:constructive}.
\end{proof}

\textbf{Proof of Proposition \ref{prop:restrict_data}:} Follows directly from combining \eqref{lemma:exist_feasible} and Corollary \eqref{corollary:feasilble_implies_optimal}.

\textbf{Proof of Theorem \ref{theorem:opt_supp}:} The proof is detailed in Appendix~\ref{appendix:theoremproof}. We will briefly provide an outline. Consider problem \eqref{eq:implbias_reg} but for a new data set $\{(x_i,y_i)\}_{i\in R_{reg}}$. By Proposition \ref{prop:restrict_data}, there exists a primal-dual optimal pair $(\mu^*, \alpha^*)$ with $\text{supp}(\mu^*) \subseteq F_{reg}$ which solves problem \eqref{eq:implbias_reg}. We remark that $\,\mu^*$ is feasible for problem \eqref{eq:implbias_reg} with the full data set $\{(x_i,y_i)\}_{i=1}^n$. It remains to show  $\exists\alpha \in\mathbb{R}^n$ which is feasible and attains strong duality with $\mu^*$. For this, we extend $\alpha^*\in\mathbb{R}^m$ to $\tilde{\alpha}\in\mathbb{R}^n$ by appending zeroes to any new entries, where $m = \left\lvert R_{reg} \right\rvert$. It is clear that $\tilde{\alpha}$ is dual-feasible, and by verifying that the pair $(\mu^*,\tilde{\alpha})$ attains strong duality we conclude. A similar reasoning applies to classification.

\section{SYNTHETIC REGRESSION TASK}\label{section:mst}

% \subsection{Toy Regression Problem}

We present a simple one-dimensional toy regression task, illustrating how the optimal supports of Theorem \ref{theorem:opt_supp} can induce the binning phenomenon. The regression data set is generated from a finite teacher neural network~$\mu_T$, which has $9$ neurons in the hidden layer. The resulting target function $f_{\mu_{T}}:[-1,1]\rightarrow [0,1]$ is the sum of two large-scale triangles and two small-scale triangles, depicted in Figure \ref{fig:mstriangle}. As usual with supervised learning problems, the task is to fit a model's parameters using a train set to obtain minimum square error on a separate validation set.

\begin{figure}[ht]
\centering
% \vspace{-2em}
\includegraphics[width=0.4\textwidth]{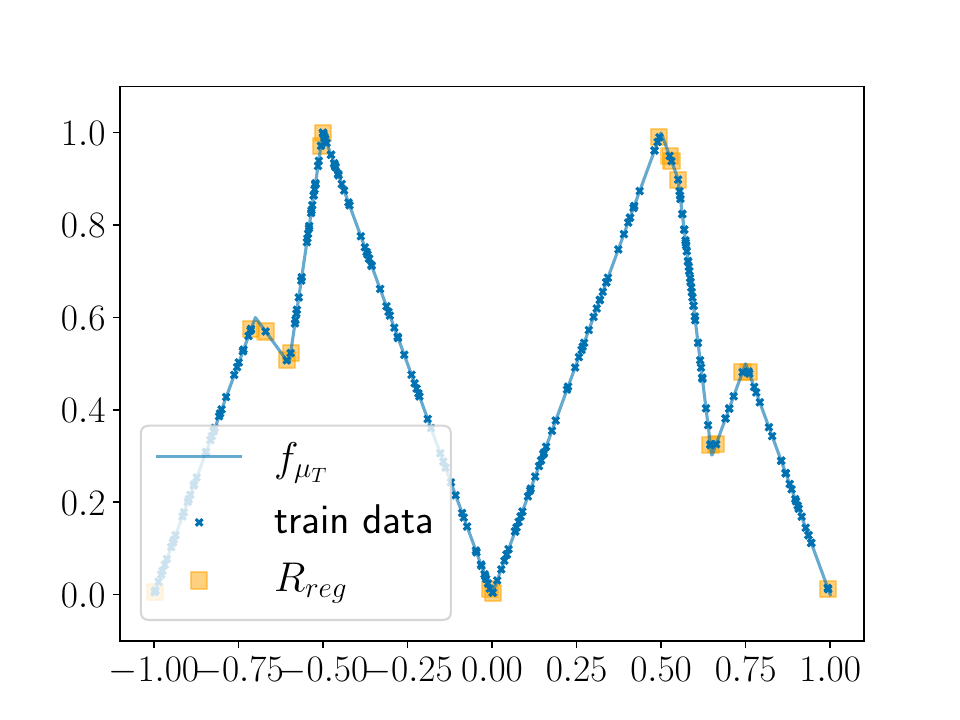}
% \vspace{.1in}
\caption{The train data set, consisting of $250$ data points $\{(x_i,f_{\mu_T}(x_i))\}_{i=1}^{250}$, is depicted by the blue crosses. $R_{reg}$ is depicted in orange, and is notably sparse, consisting of just 18 points. On the other-hand, discretizing the data with $k=50$ bins results in the set $R_{class}$ containing $230$ of the data points.}
\label{fig:mstriangle}
\end{figure}
\subsection{Experiment Setup}\label{subsection:experiment_setup}

\textbf{Data:}
To generate discrete labels, we divided the $y$-axis into $k=50$ bins of uniform size, so that the midpoint of the first/last bin was 0/1. For both the train and validation sets, we sampled $x_i$ uniformly so that each bin contained the same number of points $\left(x_i, f_{\mu_T}(x_i)\right)$. The train and validation data sets both consisted of $250$ data points.

\textbf{Models:}
We trained two over-parameterized models: \begin{enumerate}
    \item \textbf{Regression Model:} $10,000$ neurons in the hidden layer with scalar output, totalling $30,000$ weights. Trained using the square loss.\ 
    \item \textbf{Classification Model:} $500$ neurons in the hidden layer with vector output of dimension $k=50$, totalling $26,000$ weights. Trained using the cross-entropy loss.
\end{enumerate}
\textbf{Training:}
Both models were trained using gradient descent for thirty random initializations of their weights, following the scheme given by \citet{initializing}. A hyper-parameter sweep was used to find the optimal learning rate for each of the models. The stopping criterion was when neither the train nor validation losses decreased from their best observed values over a duration of 1000 epochs. The final model parameters were taken from the epoch that obtained lowest square validation error. To obtain real-valued predictions from the classification model, we took the expected value over the bins as described in Section~\ref{section:binning}.

\subsection{Results}\label{subsection:results_mst}

\begin{table}[t]
\vspace{0.5em}
\begin{center}
\begin{tabular}{|l|cccc|}
\hline
\multicolumn{1}{|c|}{}  & \multicolumn{4}{c|}{\textbf{RMSE} $\times 10^{2}$}      \\ \cline{2-5} 
\multicolumn{1}{|c|}{}  & \textbf{Best} & \textbf{Worst}  & \textbf{Mean} & \textbf{Std. Dev} \\ \hline
Regression     & 3.70  & 6.85   & 4.55     & 1.38       \\ 
Classification & 0.86 & 1.54    & 1.21     & 0.19  \\ \hline
\end{tabular}
\end{center}
\caption{Population statistics for the RMSE over 30 random initializations of model weights.}
\label{table:30runs}
\end{table}

\begin{figure}[t]
    \centering
    \includegraphics[width=0.8\linewidth]{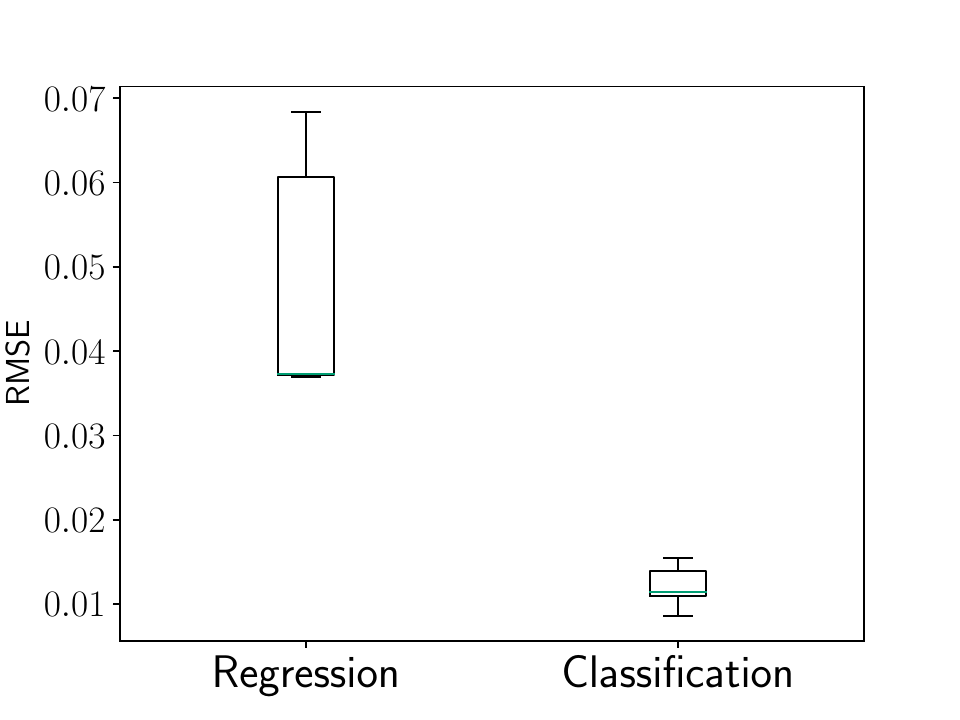}
    \caption{Population RMSE over 30 random initializations.}
    \label{fig:rmse}
\end{figure}

\begin{figure*}[t]
     \centering
     \begin{subfigure}[h]{0.24\textwidth}
         \centering
         \includegraphics[width=\textwidth]{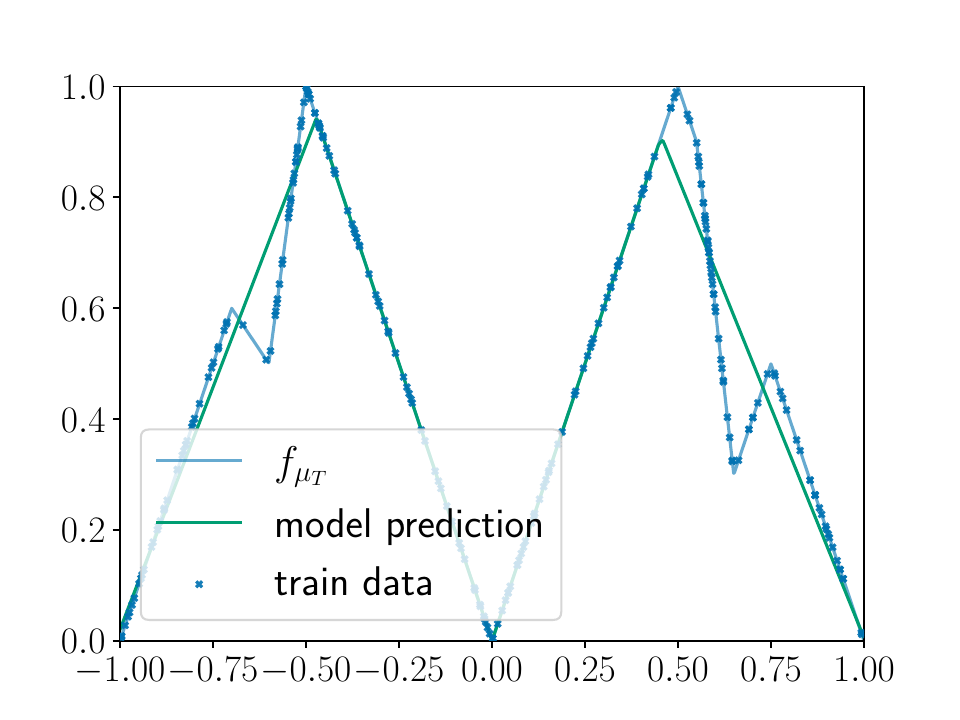}
         \caption{}
         \label{subfig:mst_reg_preds}
     \end{subfigure}
     \hfill
     \begin{subfigure}[h]{0.24\textwidth}
         \centering
         \includegraphics[width=\textwidth]{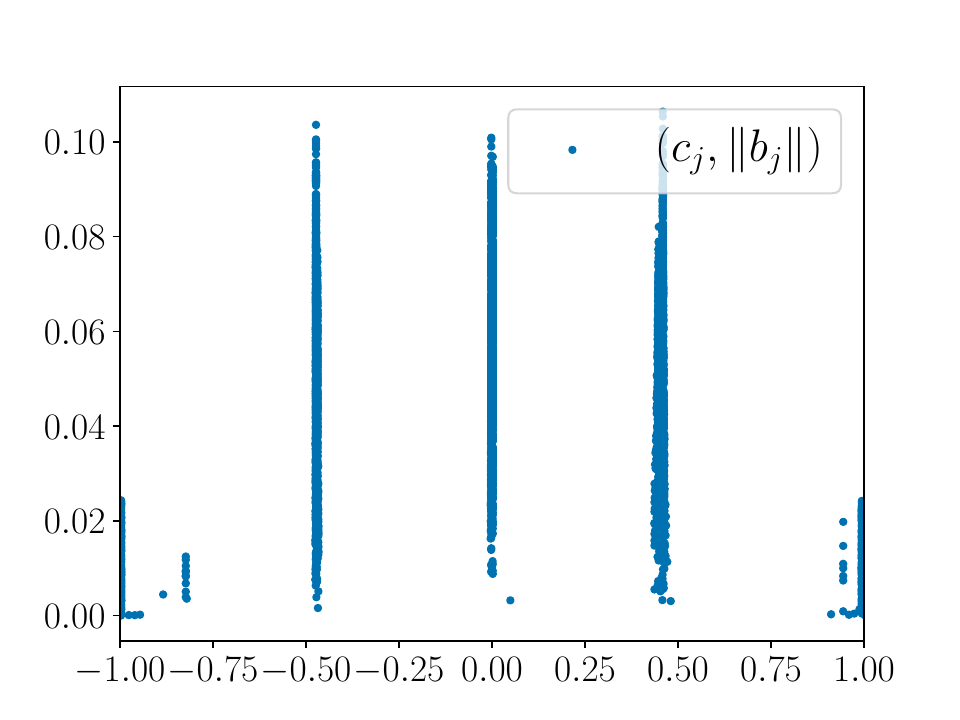}
         \caption{}
         \label{subfig:mst_reg_supp}
     \end{subfigure}
     \hfill
     \begin{subfigure}[h]{0.24\textwidth}
         \centering
         \includegraphics[width=\textwidth]{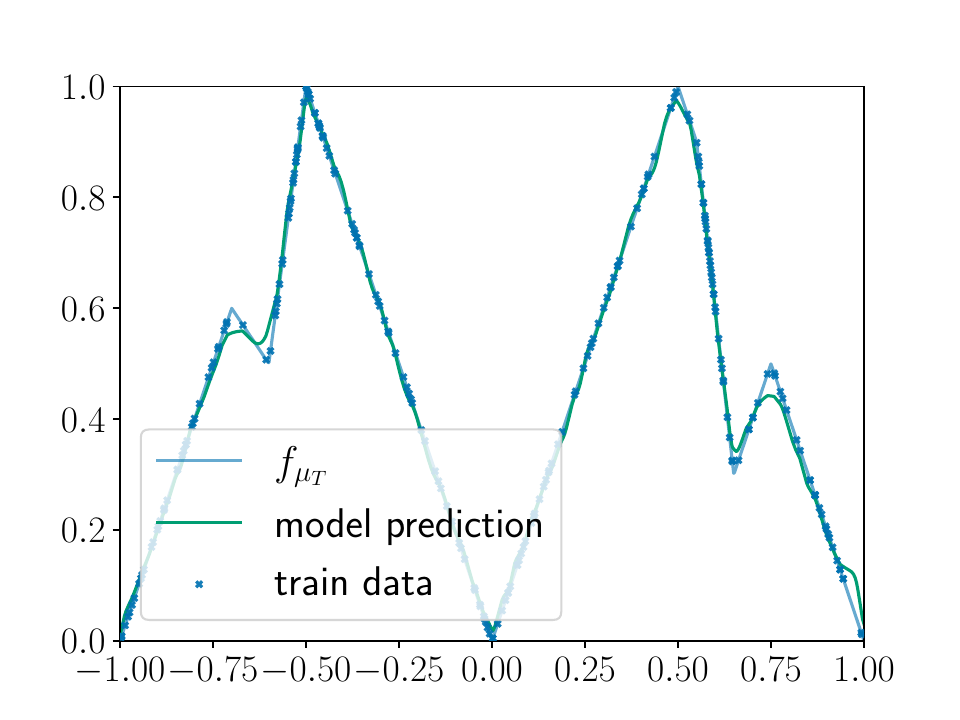}
         \caption{}
         \label{subfig:mst_class_preds}
     \end{subfigure}
     \hfill
          \begin{subfigure}[h]{0.24\textwidth}
         \centering
         \includegraphics[width=\textwidth]{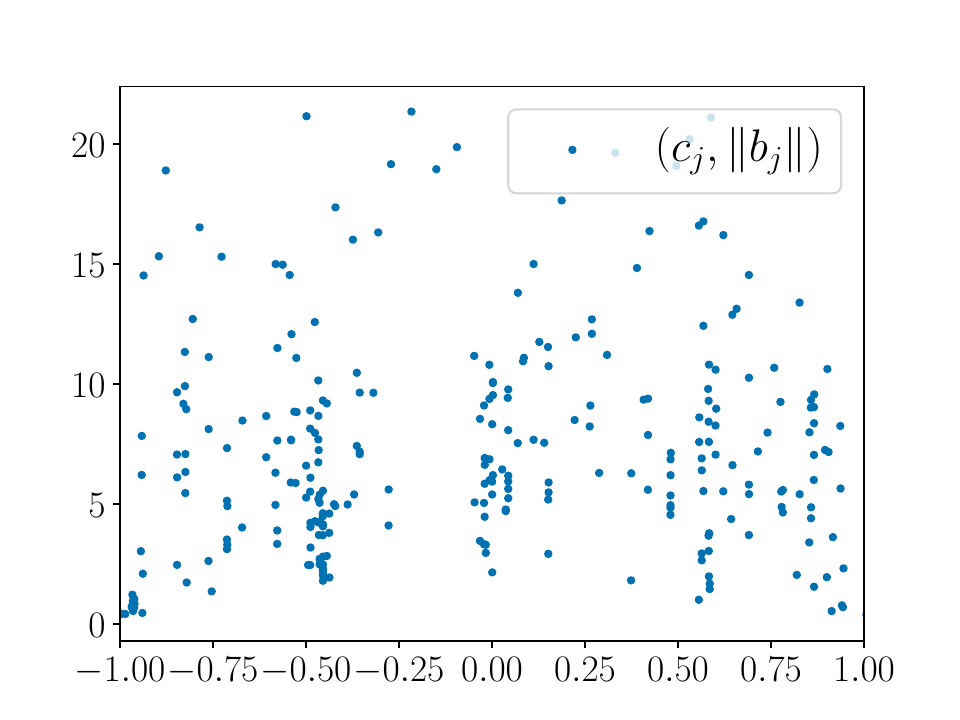}
         \caption{}
         \label{subfig:mst_class_supp}
     \end{subfigure}
        \caption{Predictions for the worst performing regression / classification models are depicted in Figures \ref{subfig:mst_reg_preds} / \ref{subfig:mst_class_preds}. Supports for the worst performing regression / classification models are depicted in Figures \ref{subfig:mst_reg_supp} / \ref{subfig:mst_class_supp}.}
        \label{fig:underfitting}
\end{figure*}

The validation RMSE for the thirty random intializations is displayed in Figure \ref{fig:rmse} and Table \ref{table:30runs}. Our regression task clearly exhibits the binning phenomenon; every classification models attained lower validation error than the best performing regression model. Moreover, the classification models were more stable to train, exhibiting less variance in performance over the thirty random initializations.

The predictions of the worst performing regression and classification models are depicted in Figures \ref{subfig:mst_reg_preds} and \ref{subfig:mst_class_preds} respectively. It can be seen that the regression model was unable to fit the smaller-scale triangles from the train data, converging to a local minima of the square loss.

Figures \ref{subfig:mst_reg_supp} and \ref{subfig:mst_class_supp} depict the kinks $c_j$ corresponding to the model's input weights $a_j$, for both the regression and classification model respectively. The x-axis depicts the position of a feature's kink, and the y-axis expresses the norm of the corresponding output-layer weight. The support of the regression model is notably sparse, with the kinks gathering at points corresponding to to the peaks of the larger-scale triangles. The model has struggled during optimization to recover all of $R_{reg}$, lacking the features whose kinks are located at peaks of the smaller-scale triangles, and as a consequence suffers under-fitting. 

On the other-hand, the classification model recovers a support which has features more evenly distributed across the unit interval, aligning with the optimal support $R_{class}$ described in Theorem \ref{theorem:opt_supp}. As a consequence, the classification model does not suffer the same optimization problem as the regression model.

\subsection{Discussion}

\citet{globalconvergence_bach} show that in the infinite width limit, the gradient flow of a two-layer neural network converges to the global minimizer of the problem. Our experiment indicates that even simple problems can result in global convergence only being guaranteed at extreme widths.

For regression data generated from a teacher network with $m_0$ hidden-neurons and Gaussian weights, \citet{spurrious} show that training a model with $m=m_0+1$ neurons helps alleviate under-fitting, postulating that increasing $m$ further aids optimization. This is a clear example where regression does not suffer under-fitting, and over-parameterization aids training. Our results indicate somewhat surprisingly that the implicit bias can play a fundamental role in gradient based optimization, even for over-parameterized models and when $m>>m_0$.

\citet{characterizing_goodfellow} demonstrate that on a straight line between the optimal parameters and a random initialization, various over-parameterized state of the art vision models encounter no local minima. We provide evidence in Appendix \ref{appendix:gradtraj} demonstrating that for even simple problems, both the regression and classification models can deviate from a linear path during optimization.

\section{IMPLICIT BIAS FOR HIGHER DIMENSIONS}\label{section:highdimex}
We provide an experiment that indicates that the properties of the supports provided in Section \ref{section:optimal} likely apply for higher dimensions. We generated regression data $(x_i,f_{\mu_T}(x_i))\in [-1,1]^3\times \mathbb{R}$ from a teacher network $\mu_T$ with three neurons and random weights, where $x_{3} = 1$. Similar to section \ref{subsection:experiment_setup}, we trained over-parameterized regression and classification models on the regression data and binned data (using $k=25$ uniform sized bins) respectively. The precise details of the experiment can be found in Appendix \ref{appendix:highdimex}.

Each feature $a_j$ is now characterized by the line where it ramps. That is to say, the points $x\in\mathbb{R}^3$ satisfying:
\begin{equation*}
    a_{j,1}x_1 + a_{j,2}x_2 + a_{j,3} = 0,
\end{equation*} where $x_3 = 1$. The critical lines characterizing the features of the regression and classification models are depicted in Figure \ref{fig:3dreg_supp} and \ref{fig:3dclass_sup}, respectively. 
We see that the regression model recovers a sparse support, whilst the classification model's features are more evenly distributed over unit square corresponding to $(x_1,x_2)$. These observations are similar to $R_{reg}$ and $R_{class}$ in the one-dimensional case, suggesting that the difference in implicit bias between regression and classification support we identified in one-dimensional problems is likely to hold in higher dimensions.

\begin{figure}[t]
     \centering
     \begin{subfigure}[h]{0.23\textwidth}
         \centering
         \includegraphics[width=\textwidth]{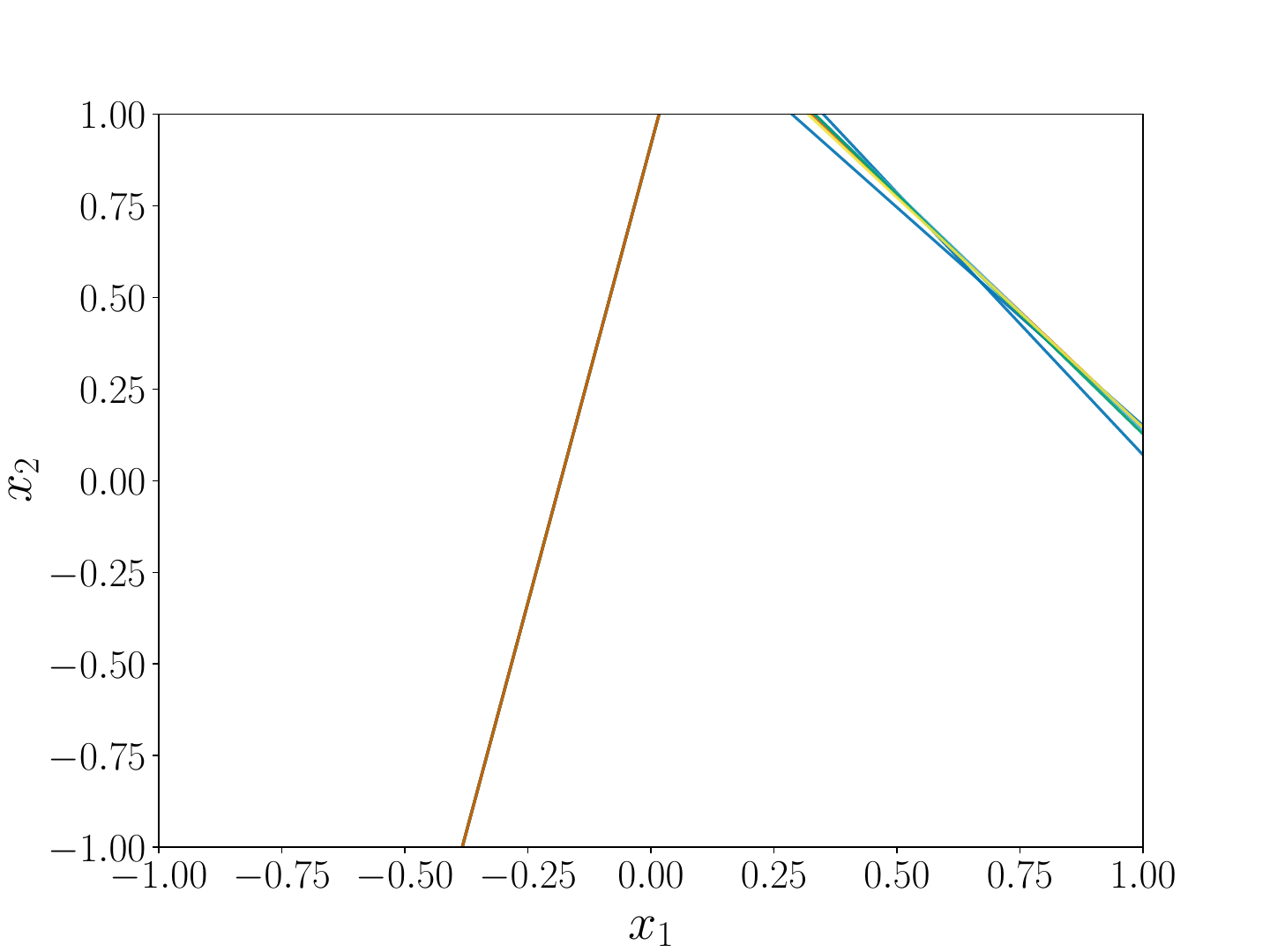}
         \caption{}
         \label{fig:3dreg_supp}
     \end{subfigure}
     \hfill
     \begin{subfigure}[h]{0.23\textwidth}
         \centering
         \includegraphics[width=\textwidth]{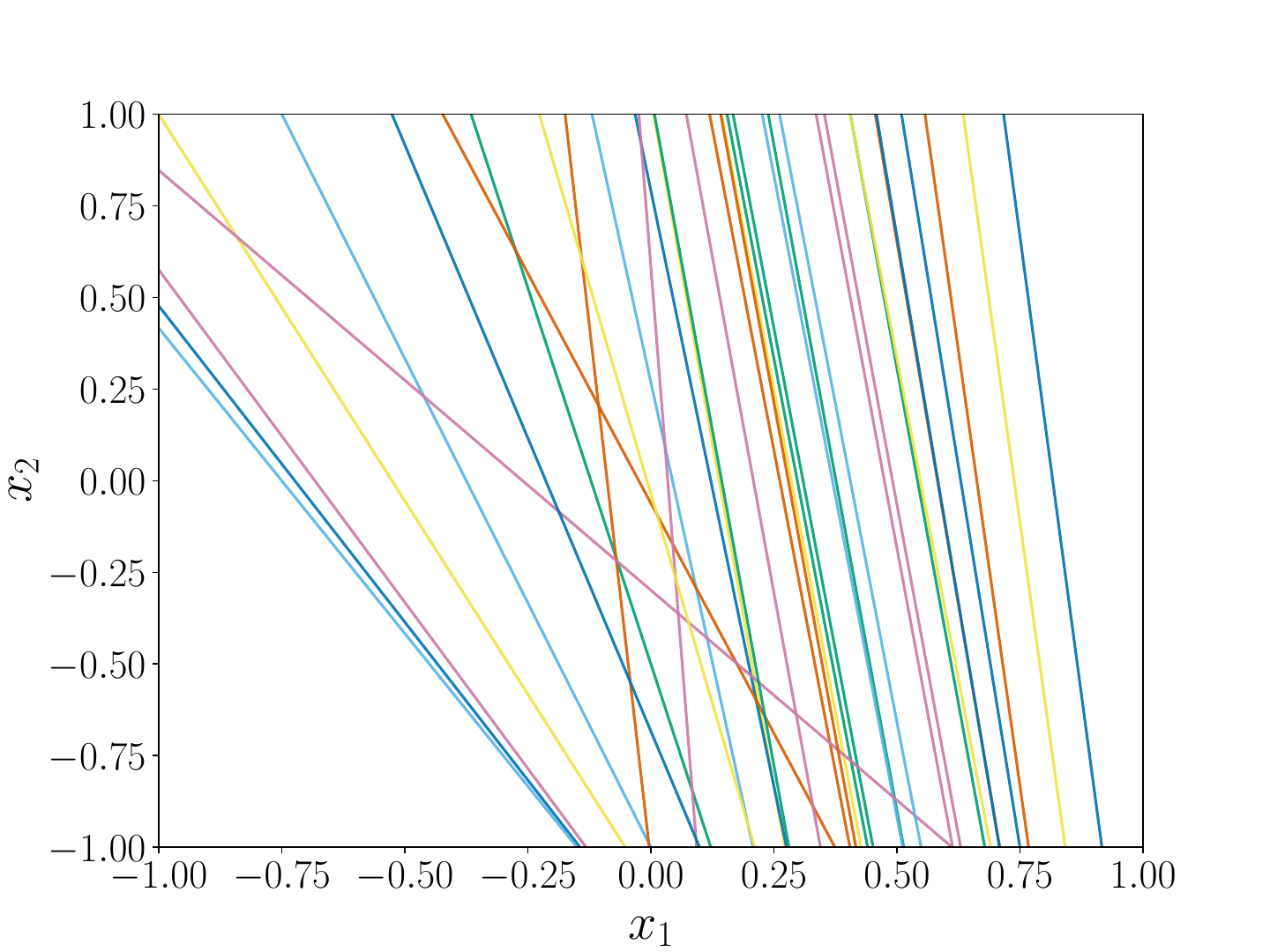}
         \caption{}
         \label{fig:3dclass_sup}
     \end{subfigure}
     \hfill
        \caption{Critical lines of the regression model's features (left) and classification model's features (right), for two-dimensional input data.}
        \label{fig:3d}
\end{figure}

\section{CONCLUSION}

We have presented supports $R_{reg}, \,R_{class}$ characterizing finite neural networks which are solutions to problems relating to known implicit biases for regression and classification, in the case of one-dimensional data. We postulated that the differences between these two supports provided one explanation for the binning phenomenon. This claim was supported by numerical experiments, demonstrating that over-parameterized models learn features which notably coincide with the supports we proposed. Moreover, our synthetic problem clearly exhibited the binning phenomenon, resulting from the inability of the regression model to recover all of the sparse optimal support during training. Finally, we provided empirical evidence that the characteristics of our proposed supports hold in higher dimensions.

As far as we are aware, the implicit biases of
arbitrarily deep neural networks is still an active research topic and is not currently known. If theorems similar to those of
\citet{classbias_bach} and \citet{regbias_boursier} are obtained for deeper models, it may indeed be possible to extend the reasoning presented in our paper to deep neural nets by inspecting layer-wise features.

Our results raise many questions, both from a practical perspective and from a theoretical stand-point.

\textbf{Practice:} For some problems, the cross-entropy loss outperforms the square loss on regression tasks, despite it having no information about relationship between classes. Future works could investigate how to best incorporate the notion of adjacency between the bins, building on existing works such as \citet{ordinal_graph_classif}.  Other directions could include exploring different ways to discretize the data (e.g., jointly learning bins of differing sizes), or how best to choose the number of bins $k$ for discretization.

\textbf{Theory:} A natural progression would be to prove a Theorem similar to that of \ref{theorem:opt_supp}, but for the implicit biases described in \citet{regbias_boursier,classbias_bach}. Another option would be to extend the results of Theorem~\ref{theorem:opt_supp} to the case of multi-dimensional data. Further works on the implicit bias of deep models could help to explain the binning phenomenon reported in the literature mentioned in Section \ref{section:intro}.

\subsection*{Acknowledgements}

 Thanks to Wilson Jallet and Nat McAleese for discussions relating to the binning phenomenon. We would also like to extend our thanks and gratitude to all of the reviewers for their constructive feedback.  We also acknowledge support from the French government under the management of the Agence Nationale de la Recherche as part of the “Investissements d’avenir” program, reference ANR-19-P3IA-0001 (PRAIRIE 3IA Institute), as well as from the European Research Council (grant SEQUOIA 724063).

\bibliography{refs}
\bibliographystyle{abbrvnat}

\newpage
\appendix

% \aistatstitle{Regression as Classification: \\
% Influence of Task Formulation on Neural Network Features\\
% Supplementary Materials}
\onecolumn
\textbf{\Large Appendix}

\section{PROOF OF IMPLICIT BIAS FOR MULTI-CLASS SOFTMAX REGRESSION}\label{appendix:bach_extension}

We consider the set-up of \citet[Theorem 3]{classbias_bach}, and can follow the exact same proof, except that now the feature map is $k$-dimensional rather than $1$-dimensional. Assumption $(A1)$ is unchanged, while Assumption $(A3)$ is considered component-wise.

For Assumption $(A2)$, we use the framework of Theorem 7 from~\citet{soudry2018implicit} instead of the one in Theorem 3.

We can then extend the informal argument from \citet{classbias_bach} that when using a predictor $\sum_{i=1}^m b_j (a_j^\top x)_+$, we converge to the minimum $\ell_2$-norms $\sum_{i=1}^m \| b_j\|_2^2 + \| a_j\|_2^2$, which is minimized by scaling invariance as $2\sum_{i=1}^m \| b_j\|_2 \| a_j\|_2$, and thus, writing the predictor as:
$$\sum_{i=1}^m b_j (a_j^\top x)_+=\sum_{i=1}^m b_j \|a_j\|_2 ((a_j/\|a_j\|_2)^\top x)_+ =\int_{S^{d-1}}(a^Tx)_+d\nu(a)  \quad \quad \text{for} \quad \nu = \sum_{j=1}^mb_j\|a_j\|\delta_{\frac{a_j}{\|a_j\|}},$$ the penalty is exactly proportional to the total variation norm with $\ell_2$-penalties.

\section{MINIMIZING TOTAL VARIATION OF A MEASURE WITH CONVEX LOWER SEMI-CONTINUOUS CONSTRAINTS}\label{appendix:derive_dual}
\begin{lemma}\label{lemma:sup_int_norm} Let 
  $\left\langle \cdot , \cdot \right\rangle$ denote the Euclidean inner-product defined on $\mathbb{R}^k$, and let $\Omega$ denote any norm on $\mathbb{R}^k$. Then for any measurable function $g:\mathbb{U} \rightarrow \mathbb{R}^k$:

\begin{equation} \sup\limits_{\mu\in\mathcal{M}(\mathbb{U},\mathbb{R}^k)} \int_\mathbb{U} \Big( \left\langle
  g(u), d\mu(u) \right\rangle - \Omega\left( d\mu(u) \right) \Big) \quad = \quad
  \begin{cases}
  0 \quad &\text{if} \quad \Omega_*\left( g(u) \right) \leq 1 \quad \forall u \in
    \mathbb{U}\\ \infty \quad & \text{otherwise.}
  \end{cases}\end{equation}\end{lemma}

\begin{proof} Suppose that $\exists u_0 \in \mathbb{U}$ such that $\Omega_*\left( g(u_0) \right)
  > 1$. By definition of the dual norm, this implies:

  \begin{equation*}   \left\langle v_0, g(u_0) \right\rangle =   \sup\limits_{\Omega( v
  ) =1} \left\langle v, g(u_0) \right\rangle > 1, \end{equation*}

  where we have used $v_0 \in \mathbb{R}^k$ to denote the vector that attains the supremum. For $t>0$, consider the measure $\mu = tv_0 \delta_{u_0} \in \mathcal{M}(\mathbb{U}, \mathbb{R}^k)$. Then:

  \begin{align*} \int_\mathbb{U} \Big( \left\langle
  g(u), d\mu(u) \right\rangle - \Omega\left( d\mu(u) \right) \Big) &= \left\langle g(u_0), tv_0\right\rangle - \Omega\left( tv_0
  \right) \\ &= t \left( \underbrace{ \left\langle g(u_0) ,v_0 \right\rangle - \Omega\left( v_0
\right)}_{>1} \right).\end{align*}

Hence taking the limit as $t\rightarrow \infty$ leads to an unbounded
supremum. Conversely, suppose that $\Omega_* \left( g(u) \right) \leq 1 \quad \forall u \in
\mathbb{U}$. Then for any $u\in \mathbb{U}$ and $v\in \mathbb{R}^k$ one has: 

\begin{align*} \left\langle g(u), v \right\rangle &= \Omega (v) \, \left\langle g(u),
\frac{v}{\Omega (v)} \right\rangle \\ &\leq \Omega (v)
\sup\limits_{\Omega (w) = 1} \left\langle g(u), w \right\rangle \\ &= \Omega(v) \, \Omega_* \left( g(u) \right) \leq \Omega (v). \end{align*}

We conclude that $\int_\mathbb{U} \Big( \left\langle
  g(u), d\mu(u) \right\rangle - \Omega\left( d\mu(u) \right) \Big)\leq 0$, and hence:

\begin{equation*} \Omega \left( g(u) \right)_* \leq
1 \quad \forall a\in A \Longrightarrow \sup\limits_{\mu\in\mathcal{M}(\mathbb{U}, \mathbb{R}^k)}
\int_\mathbb{U} \Big( \left\langle
  g(u), d\mu(u) \right\rangle - \Omega\left( d\mu(u) \right) \Big)
\quad = \quad 0. \end{equation*}

\end{proof}

\begin{lemma}

The dual of problem \eqref{eq:general_mintv} is:
  \begin{equation*} \begin{aligned} & \underset{\alpha_1\ldots,\alpha_n \in \mathbb{R}^{k}}{\sup} & &
    -\sum_{i=1}^n \sigma_{S_i}(\alpha_i)\\ & \text{subject to} & & \Omega_*\left(
    \sum_{i=1}^n \alpha_i \phi_u(x_i) \right)\leq
1 \quad \forall u\in \mathbb{U}. \end{aligned} \end{equation*}
\end{lemma}

\begin{proof}
By Fenchel duality \citep{bidual} one has:
    
    \begin{equation*} I_{S_i}^{**}\left(f_\mu(x_i)\right) = I_{S_i}\left(f_\mu(x_i)\right) = \sup\limits_{\alpha_i \in
    \mathbb{R}^k} \: \left\{ \langle \alpha_i, f_\mu(x_i) \rangle -
\sigma_{S_i}(\alpha_i) \right\}. \end{equation*} 

  Plugging this into the Lagrangian we obtain: 

  \begin{align*} L(\mu) &= \int_{\mathbb{U}} \Omega\left( d\mu(u) \right) \quad +
   \quad  \sum_{i=1}^n\sup\limits_{\alpha_i \in \mathbb{R}^k} \: \left\{ \langle
    \alpha_i, f_\mu(x_i) \rangle - \sigma_{S_i}(\alpha_i) \right\} \\ &=  \int_{\mathbb{U}}
    \Omega\left( d\mu(u) \right) \quad +  \quad \sup\limits_{\alpha_1,\ldots,\alpha_n \in \mathbb{R}^k}
    \sum_{i=1}^n \: \left\{ \langle \alpha_i, f_\mu(x_i) \rangle -
    \sigma_{S_i}(\alpha_i) \right\} \\ &=  \int_{\mathbb{U}} \Omega\left( d\mu(u) \right) \quad + \quad 
    \sup\limits_{\alpha_1,\ldots,\alpha_n \in \mathbb{R}^k} \left\{  \int_{\mathbb{U}} \big\langle
      \sum_{i=1}^n\ \alpha_i \phi_u(x_i), d\mu(u) \big\rangle  \quad - \:
      \sum_{i=1}^n \sigma_{S_i}(\alpha_i) \right\}\\ &= \sup\limits_{\alpha_1,\ldots, \alpha_n \in\mathbb{R}^k} \left \{ - \int_{\mathbb{U}} \Big( \big\langle
\underbrace{-\sum_{i=1}^n\ \alpha_i \phi_u(x_i)}_{g(u)}, d\mu(u) \big\rangle \:
- \: \Omega\left( d\mu(u) \right) \Big) \quad -  \sum_{i=1}^n \sigma_{S_i}(\alpha_i) \right\}.
    \end{align*}

The primal value $\inf\limits_{\mu\in\mathcal{M}(\mathbb{U}, \mathbb{R}^k)} L(\mu)$ is hence:

\begin{equation*}  \inf\limits_{\mu\in\mathcal{M}(\mathbb{U}, \mathbb{R}^k)} \:\: \sup\limits_{\alpha_1,\ldots,\alpha_n \in \mathbb{R}^k} \left \{ - \int_{\mathbb{U}} \Big( \big\langle g(u), d\mu(u) \big\rangle \: - \:
\Omega\left( d\mu(u) \right) \Big) \quad - \sum_{i=1}^n \sigma_{S_i}(\alpha_i) \right\}.
\end{equation*}

The dual problem is obtained by switching the order of the supremum and
infinum: 

    \begin{align*} &   \sup\limits_{\alpha_1,\ldots,\alpha_n \in \mathbb{R}^k}
      \:\:\inf\limits_{\mu\in\mathcal{M}(\mathbb{U}, \mathbb{R}^k)} \left \{ - \int_{\mathbb{U}} \Big( \big\langle
      g(u), d\mu(u) \big\rangle \: - \: \Omega\left( d\mu(u) \right) \Big) \quad -
      \sum_{i=1}^n \sigma_{S_i}(\alpha_i) \right\} \\  = & \sup\limits_{\alpha_1,\ldots,\alpha_n \in \mathbb{R}^k} \left\{ \quad \inf\limits_{\mu\in\mathcal{M}(\mathbb{U}, \mathbb{R}^k)} \Big\{ - \int_{\mathbb{U}} \Big(
\big\langle g(u), d\mu(u) \big\rangle \: - \: \Omega\left( d\mu(u) \right) \Big)
\:\:\Big\} \quad  - \sum_{i=1}^n \sigma_{S_i}(\alpha_i) \right\}. \end{align*}

Applying Lemma \ref{lemma:sup_int_norm} we conclude the dual problem is:

  \begin{equation*} \begin{aligned} & \underset{\alpha_1\ldots,\alpha_n \in \mathbb{R}^{k}}{\sup} & &
    -\sum_{i=1}^n \sigma_{S_i}(\alpha_i)\\ & \text{subject to} & & \Omega_*\left(
    \sum_{i=1}^n \alpha_i \phi_u(x_i) \right)\leq
1 \quad \forall u\in \mathbb{U}. \end{aligned} \end{equation*}

 \end{proof} 

\section{PROOF OF LEMMA \ref{lemma:sufficient_dualfeasible}}\label{appendix:dual_Ux}
% \ls{Make clearer and change to $c$.} 

Define:
\begin{align*}
    \mathbb{U}^+ &= \{1\} \times [-1,1] \\
    \mathbb{U}^- &= \{-1\} \times [-1,1],
\end{align*}
and note that $\mathbb{U} = \mathbb{U}^+ \cup \mathbb{U}^-$. Similarly we define:
\begin{align*}
    {U_X}^+ &= \{1\} \times\{x_1,\ldots,x_n\} \\
   {U_X}^- &= \{-1\} \times\{x_1,\ldots,x_n\}.
\end{align*}

Firstly, let us show that:

\begin{equation}\label{eq:rightfeats}
  \Omega_*\left(
    \sum_{i=1}^n \alpha_i \phi_u(x_i) \right)\leq
1 \quad \forall u\in \mathbb{U}^+.
\end{equation}

This is equivalent to showing that:
\begin{equation}\label{eq:c_rightfeats}
  \Omega_*\left(
    \sum_{i=1}^n \alpha_i (x_i-c) \right)\leq
1 \quad \forall c\in[-1,1].
\end{equation}

By assumption, we know that:
\begin{align*}
  &\Omega_*\left(
    \sum_{i=1}^n \alpha_i \phi_u(x_i) \right)\leq
1 \quad \forall u\in {U_X}^+ \\
\Longleftrightarrow \quad &\Omega_*\left(
    \sum_{i=1}^n \alpha_i (x_i -c)_+ \right)\leq
1 \quad \forall c\in\{x_1,\ldots,x_n\}.
\end{align*}

 Let $g(c)  = \sum_{i=1}^n \alpha_i (x_i-c)_+$ and remark that $g$ is a piece-wise affine function with line segments meeting at $\{x_1, \ldots, x_n\}$. As a consequence, $\forall c\in[-1,1]$, $\exists\theta \in [0,1]$ and $i\in [n-1]$ such that:\begin{equation*}
    g(c) =  \theta \, g(x_i) + (1-\theta) \, g(x_{i+1}).
\end{equation*}
By the convexity of the dual norm, we have:
\begin{equation*}
    \Omega_*\left(g(c)\right) \leq \theta\, \underbrace{\Omega_*\left(g(x_i)\right)}_{\leq1} + (1-\theta)\,\underbrace{\Omega_*\left(g(x_{i+1})\right)}_{\leq 1} \leq 1.
\end{equation*}
We conclude that \eqref{eq:c_rightfeats} (and hence \eqref{eq:rightfeats}) hold. We conclude by repeating the same argument for $\mathbb{U}^-$, but replacing $g(c)$ with $h(c) = \sum_{i=1}^n \alpha_i(c-x_i)_+$.

\section{CONSTRUCTION OF FEASIBLE MEASURES}\label{appendix:constructive}

\subsection{REGRESSION}\label{appendix:reg_feasproof}
\begin{proof}

\begin{figure}[ht]
    \centering
    \includegraphics[width=0.4\linewidth]{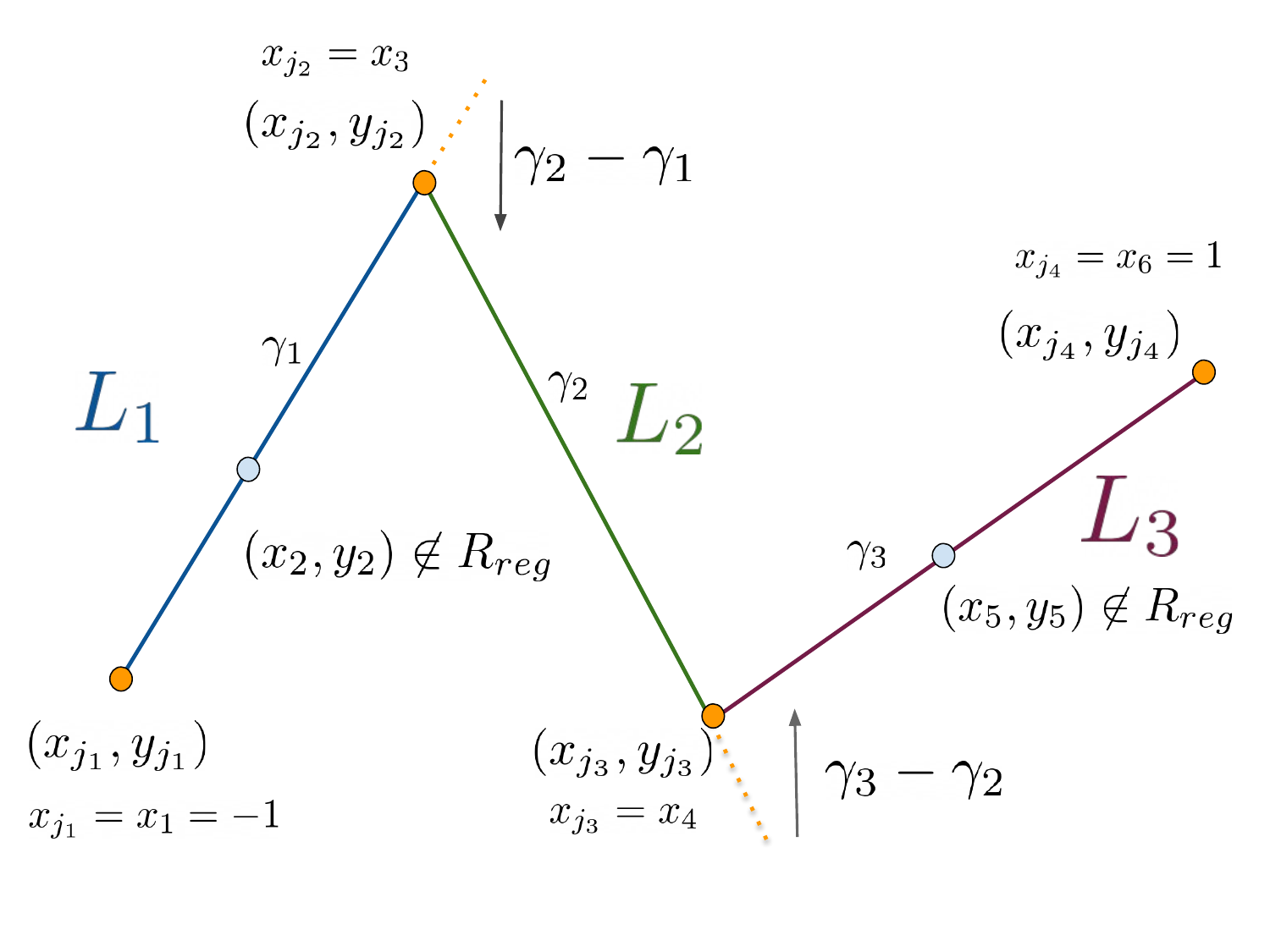}
    \vspace{-0.5em}
    \caption{Graphical depiction of $R_{reg}$ for six arbitrary data points $(x_1,y_1),\ldots, (x_6,y_6)$. In this example $J_{reg}=\{1,3,4,6\}$, resulting in a piece-wise interpolant with three line components $L_1, L_2$ and $L_3$.}
    \label{fig:reg_diag}
\end{figure}

As in the statement let $-1=x_1<\cdots< x_n=1$ and $y_1,\cdots, y_n \in \mathbb{R}$ denote the regression data. We define:

\begin{equation} J_{reg} =\{i\in [n] \: : \: x_i \in R_{reg}\},\end{equation} 

as the set of indices $i$ which correspond to data points $x_i \in R_{reg}$. Let $m=\left\lvert J_{reg}\right\rvert$, and note that $\{1,n\}\subseteq J_{reg} \Longrightarrow m \geq 2$. Without loss of generality we assume that the elements of $J_{reg}$ are sorted in increasing order $1=j_1<\cdots<j_m=n$. 

For $l\in [m-1]$, let $L_l$ denote the line passing through $\left(x_{j_l},y_{j_l}\right)$ and $\left(x_{j_{l+1}},y_{j_{l+1}}\right)$. By definition, the equation of each of the $m-1$ lines will be:

\begin{equation}
        L_l(x) = \gamma_l (x-x_{j_l}) + y_{j_l},
\end{equation}

where:

\begin{equation}
\gamma_l = \frac{y_{j_{l+1}} - y_{j_l}}{x_{j_{l+1}} - x_{j_l}}.
\end{equation}

A graphical depiction of this above notation can be found in Figure \ref{fig:reg_diag}. Finally, we write $P$ to denote the piece-wise linear interpolant of the data $\{(x_i,y_i)\}_{i=1}^n$, where:

\begin{equation}
 x\in [x_{j_l}, x_{j_{l+1}}] \quad \Longrightarrow \quad P(x) = L_l(x) \quad \quad  \forall l\in[m-1].
\end{equation}

It is sufficient to construct a measure $\mu\in\mathcal{M}(U_X, \mathbb{R})$ such that $f_\mu = P$, since then $f_\mu(x_i)= P(x_i) = y_i \quad \forall i\in[n]$.

We begin by first constructing a measure $\mu_1$ such that $f_{\mu_1}(x) = L_1(x)$, which we will later build upon to construct $\mu$.  From the definition of $R_{reg}$, the measure $\mu_1$ will satisfy $f_{\mu_1}(x_i)=y_i$ for all $i\in [j_2]$.  

Let $u_1,u_n\in U_X$ be defined as $u_1 = (1,x_1)$ and $u_n = (-1,x_n)$. We claim that $\mu_1$ can be written in the following form:

\begin{equation}
    \mu_1 = w_l \delta_{u_1} + w_r \delta_{u_n}.
\end{equation}

To show this is true, we search for weights $w_r,w_l \in \mathbb{R}$ that satisfy:

\begin{align*}
   L_1(x) = \gamma_1(x-x_1) + y_1 &=  f_{\mu_1}(x) \\  &= \int_\mathbb{U} \phi_u(x) \, d\mu_1(u) \\ &= w_l\phi_{u_1}(x) + w_r\phi_{u_n}(x) \\
   &= w_l (x-x_1)_+ + w_r (x_n - x)_+ \\
   &= w_l (x - (-1))  + w_r (1 - x)_+ \\
   &= w_l (x+1)_+ + w_r(1-x)_+.
\end{align*}

Substituting $x = 1$ gives:

\begin{align*}
   w_l(2)_+ + w_r(0)_+   &= 2\gamma_1 + y_1 \\ 
\Longrightarrow w_l = \gamma_1 + \frac{y_1}{2}.
\end{align*}

Similarly, substituting $x = -1$ gives:

\begin{align*}
   w_l(0)_+ + w_r(2)_+   &= 0\gamma_1 + y_1 \\ 
\Longrightarrow w_r = \frac{y_1}{2}.
\end{align*}

In the case that $m=2$ we are done, since $f_{\mu_1} = P$. Otherwise, consider the measure:

% If $m>2$, let $L_2$ be the line passing through the points $(x_{j_2},y_{j_2})$ and $(x_{j_3},y_{j_3})$. Consider the measure:

% \begin{equation}
%     \mu_2 = \mu_1 + \left(\gamma_{j_3} - \gamma_{j_2}\right)\delta_{u_{j_2}},
% \end{equation}

% where $u_{j_2} = (1, x_{j_2}) \in U_X$. 

% From the definition of $L_2$ and $\gamma_1,\gamma_2$, we have:
% \begin{enumerate}
%     \item $f_{\mu_2}(x) = f_{\mu_1}(x) = L_1(x)$ for all $x\in [x_{j_1}, x_{j_2}]$. \
%     \item $f_{\mu_2}(x) = L_2(x)$ for all $x\in [x_{j_2}, x_{j_3}]$.\
%     \item $\Longrightarrow f_{\mu_2}(x_i) = y_i$ for all $i\in [j_3]$
% \end{enumerate}

% If $m=3$ we are done. Otherwise, we can inductively repeart the above process by continuing to add  $(\gamma_{l+1} - \gamma_l) \, \delta_{u_{j_l}}$, where $u_{j_l} = (1, x_{j_l})$, until reaching $l = {m-1}$:

\begin{equation}
    \mu = \underbrace{ \frac{y_1}{2}\delta_{u_n} + \left(\frac{y_1}{2}+\gamma_1\right)\delta_{u_1}}_{\mu_1} + \sum_{l=2}^{m-1} (\gamma_{l+1} - \gamma_l) \, \delta_{u_{j_l}} \quad \in \mathcal{M}(U_X,\mathbb{R}),
\end{equation}

where $u_{j_l} = (1,x_{j_l} ) $. We claim that $f_\mu = P$, which is equivalent to saying that:
\begin{equation}
    \forall l \in [m-1], \quad f_\mu(x) = L_l(x) \quad \forall x\in [x_{j_l},  x_{j_{l+1}}].
\end{equation}
We will show by induction that the above statement holds. The base case $l=1$ is immediately verified, as $f_\mu(x) = L_1(x) \quad \forall x\in [x_{j_1}, x_{j_2}]$. To prove the inductive step, suppose for some $q \in [m-1]$ that the following holds:

\begin{equation}
    \forall l \in [q-1], \quad f_\mu(x) = L_l(x) \quad \forall x\in [x_{j_l},  x_{j_{l+1}}].
\end{equation}

We need to show that:

\begin{equation}
    \forall l \in [q], \quad f_\mu(x) = L_l(x) \quad \forall x\in [x_{j_l},  x_{j_{l+1}}].
\end{equation}

To do this, we note that:

\begin{equation*}
\begin{aligned}
    x\in [x_{j_{q-1}}, \, x_{j_{q}}] \quad  \Longrightarrow \quad f_\mu(x) &= L_{q-1}(x)  + \left(\gamma_{q} - \gamma_{q-1}\right) \, \left(x-x_{j_{q-1}}\right)_+ \\ 
    &= L_{q}(x).
\end{aligned}
\end{equation*}

\end{proof}

\subsection{CLASSIFICATION}
\begin{proof}
    
As in the statement let $-1=x_1<\cdots< x_n=1$ and $y_1,\ldots, y_n \in [k]$ denote the classification data. We define the one-hot labels $Z\in\{0,1\}^{n\times k}$ as:

\begin{equation}
    Z_{i,l} = \begin{cases}
        1 \quad &\text{if} \quad y_i = l \\
        0 \quad &\text{otherwise.}
    \end{cases}
\end{equation}

We define $k$ regression data-sets $D_1, \dots, D_k$, where:

\begin{equation}
    D_l = \left\{ (x_i,Z_{i,l}) \right\}_{i=1}^n \quad \forall l \in [k].
\end{equation}

Applying the result obtained from Appendix \ref{appendix:reg_feasproof}, $\exists \mu_1, \ldots, \mu_k \in \mathcal{M}(U_X, \mathbb{R})$ such that:

\begin{equation}
    f_{\mu_l}(x_i) = Z_{i,l} \quad \forall i\in [n] \quad \forall l\in [k].
\end{equation}

 By definition of $Z$, we conclude that $\mu = (\mu_1, \ldots, \mu_k) \in \mathcal{M}(U_X, \mathbb{R}^k)$ is feasible for problem \eqref{eq:implbias_class}.

\end{proof}

\section{PROOF OF THEOREM \ref{theorem:opt_supp}}\label{appendix:theoremproof}
\subsection{REGRESSION}

Let $J_{reg} = \{i\in[n] \, : \, x_i \in R_{reg} \}$ and let $m=\left\lvert J_{reg} \right\rvert$. 

Consider the following problem:

\begin{equation}\label{eq:dummyprob_reg} \begin{aligned} & \inf\limits_{\mu\in\mathcal{M}(\mathbb{U}, \mathbb{R})} &&\int_\mathbb{U} \left\lvert d\mu(u) \right\rvert \\ & \text{subject to}   &&f_\mu(x_j) = y_j \quad \forall j \in J_{reg}.\end{aligned} \end{equation}

 Let $P_1$ and $P_2$ be the primal values for problems \eqref{eq:implbias_reg} and \eqref{eq:dummyprob_reg} respectively. As problem \eqref{eq:dummyprob_reg} has less constraints than problem \eqref{eq:implbias_reg}, we can remark that $P_2 \leq P_1$. By Proposition \ref{prop:restrict_data}, $\exists \mu^* \in \mathcal{M}(U_X,\mathbb{R})$ and ${\alpha_1}^*, \dots , {\alpha_{m}}^* \in \mathbb{R}$ optimal for problem \eqref{eq:dummyprob_reg} satisfying $\text{supp}(\mu^*) \subseteq F_{reg}$.

Assume without loss of generality that $J_{reg}$ is ordered, with $1=j_1<\cdots<j_m= n$. For $i\in J_{reg}$, let $\psi(i)\in [m]$ denote the position of $i$ in the ordered list $j_1,\ldots,j_m$. We construct ${\tilde{\alpha}}_1,\dots,{\tilde{\alpha}}_n \in \mathbb{R}$ as follows:
\begin{equation}
    \tilde{\alpha}_i = \begin{cases} 0 \quad &\text{if} \quad i\not\in R_{reg} \\ \alpha^*_{\psi(i)} \quad & \text{if} \quad i\in R_{reg}. \end{cases} 
\end{equation}

${\alpha^*}_1,\dots,{\alpha^*}_m$ correspond to the $m$ data points $(x_{j_1},y_{j_1}), \dots, (x_{j_m}, y_{j_m})$. Our constructed ${\tilde{\alpha}}_1,\dots,{\tilde{\alpha}}_n$ is the extension of the above to all of the train data $(x_1,y_1),\dots, (x_n,y_n)$, where $i\not\in R_{reg} \Longrightarrow {\tilde{\alpha}}_i = 0$.

By construction, $\mu^*$ and ${\tilde{\alpha}}_1, \ldots, {\tilde{\alpha_n}}$ attain strong duality for \eqref{eq:implbias_reg}. However, it remains to verify that they are indeed prime and dual feasible for problem \eqref{eq:implbias_reg} respectively. For this, it is enough to verify that:
 
\begin{equation*}
    -\sigma_{S_i}(\tilde{\alpha}_i) = I_{S_i}\left(f_{\mu^*}(x_i)\right) \quad \forall i\in [n] \setminus J_{reg},
\end{equation*}

where the sets $S_i$ are those corresponding to regression, described in Section \ref{section:generalprob}. We remark that $f_{\mu^*}$ is a piece-wise affine function with line segments meeting at $\{(x_j,y_j)\}_{j\in J_{reg}}$. By definition, if $i\not\in J_{reg}$ then $f_{\mu^*}(x_i) =y_i \Longrightarrow I_{S_i}\left(f_{\mu^*}(x_i)\right) = 0$. Finally,  $\sigma_{S_i}(\tilde{\alpha}_i) = \sigma_{S_i}(0) = 0 \quad \forall i\in [n] \setminus J_{reg}$.

\subsection{CLASSIFICATION}

Let $J_{class} = \{i\in[n] \, : \, x_i \in R_{class} \}$ and let $m=\left\lvert J_{class} \right\rvert$. 

Consider the following problem:

  \begin{equation}\label{eq:dummyprob_class} \begin{aligned} & \inf\limits_{\mu\in\mathcal{M}(\mathbb{U}, \mathbb{R}^k)} &&\int_{\mathbb{U}} \:\left\lVert d\mu(u) \right\rVert  & \\ & \text{subject to}   &&(e_{y_j} - e_l)^T f_\mu(x_j) \geq  \mathds{1}(y_j\not=l), \\ & && \forall j\in J_{class}, \quad \forall l\in[k].\end{aligned} \end{equation}

 Let $P_1$ and $P_2$ be the primal values for problems \eqref{eq:implbias_class} and \eqref{eq:dummyprob_class} respectively. As problem \eqref{eq:dummyprob_class} has less constraints than problem \eqref{eq:implbias_class}, we can remark that $P_2 \leq P_1$. By Proposition \ref{prop:restrict_data}, $\exists \mu^* \in \mathcal{M}(U_X,\mathbb{R}^k)$ and ${\alpha_1}^*, \dots , {\alpha_{m}}^* \in \mathbb{R}^k$ optimal for problem \eqref{eq:dummyprob_class} satisfying $\text{supp}(\mu^*) \subseteq F_{class}$.

Assume without loss of generality that $J_{class}$ is ordered, with $1=j_1<\ldots<j_m= n$. For $i\in J_{class}$, let $\psi(i)\in [m]$ denote the position of $i$ in the ordered list $j_1,\ldots,j_m$. We construct ${\tilde{\alpha}}_1,\dots,{\tilde{\alpha}}_n \in \mathbb{R}^k$ as follows:
\begin{equation}
    \tilde{\alpha}_i = \begin{cases} \mathbf{0} \quad &\text{if} \quad i\not\in R_{class} \\ \alpha^*_{\psi(i)} \quad & \text{if} \quad i\in R_{class}. \end{cases} 
\end{equation}

$\alpha_1^*,\dots,\alpha_m^*$ correspond to the $m$ data points $(x_{j_1},y_{j_1}), \dots, (x_{j_m}, y_{j_m})$. Our constructed ${\tilde{\alpha}}_1,\dots,{\tilde{\alpha}}_n$ is the extension of the above to all of the train data $(x_1,y_1),\dots, (x_n,y_n)$, where $i\not\in R_{class} \Longrightarrow {\tilde{\alpha}}_i = \mathbf{0}$.

By construction, $\mu^*$ and ${\tilde{\alpha}}_1, \ldots, {\tilde{\alpha_n}}$ attain strong duality for \eqref{eq:implbias_reg}. However, it remains to verify that they are indeed prime and dual feasible for problem \eqref{eq:implbias_reg} respectively. For this, it is enough to verify that:
 
\begin{equation*}
    -\sigma_{S_i}(\tilde{\alpha}_i) = I_{S_i}\left(f_{\mu^*}(x_i)\right) \quad \forall i\in [n] \setminus J_{class},
\end{equation*}

 where the sets $S_i$ are those corresponding to classification, described in Section \ref{section:generalprob}. 
 
 We begin by noting that $\forall l\in[k]$,  $f_{\mu^*}(\cdot)_k$ is a piece-wise affine function with line segments meeting at points contained in some subset of $\{(x_j,y_j)\}_{j\in J_{class}}$. By definition:

 \begin{equation*}
 \begin{aligned}
     i\not\in J_{class}\quad  &\Longrightarrow\quad f_{\mu^*}(x_i)^T(e_{y_i} - e_l) \geq \mathds{1}(y_i\not=1) \quad \quad \forall l \in [k] \\[0.75em] 
     &\Longrightarrow I_{S_i}\left(f_\mu(x_i)\right) = 0 \quad \quad \forall i\in [n] \setminus J_{class}
 \end{aligned}
 \end{equation*}
 
 Finally,  $\sigma_{S_i}(\tilde{\alpha}_i) = \sigma_{S_i}(\mathbf{0}) = 0 \quad \forall i\in [n] \setminus J_{class}$.

\section{TRAJECTORY OF GRADIENT DESCENT}\label{appendix:gradtraj}

Figures \ref{subfig:grad_traj_reg} / \ref{subfig:grad_traj_class} depict the angles formed between the first 1000 gradients obtained during training for the worst performing regression / classification models.  Despite both models being over-parameterized for the problem, it is clear that the optimization route for the models was not a straight line. Similar results were seen over all of the thirty random intializations.

\begin{figure}[ht]
     \centering
     \begin{subfigure}[ht]{0.49\textwidth}
         \centering
         \includegraphics[width=0.7\textwidth]{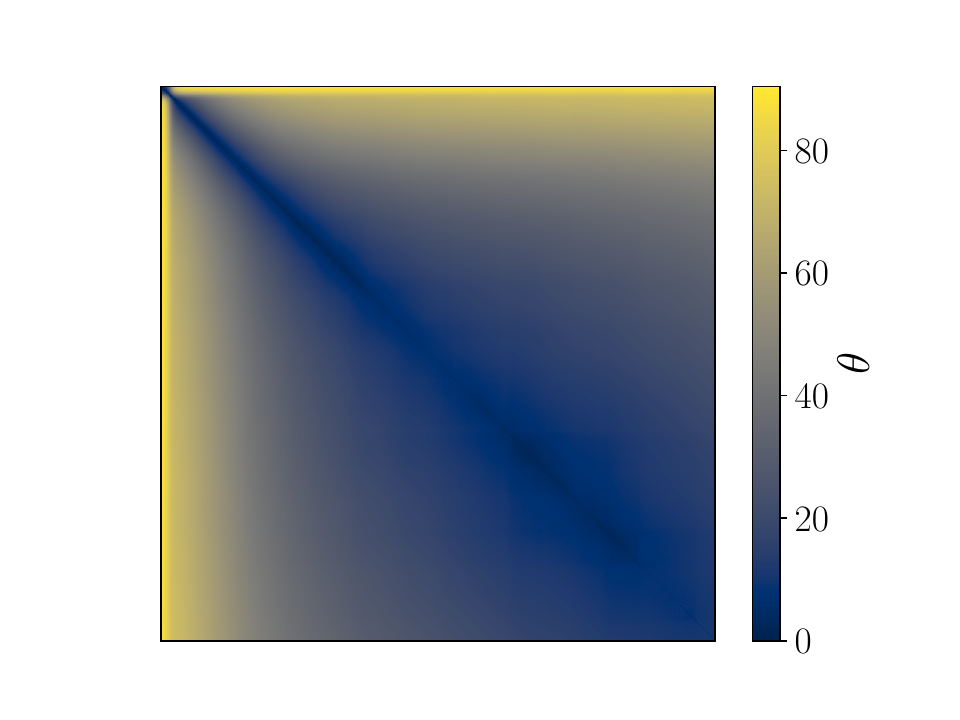}
         \caption{}
         \label{subfig:grad_traj_reg}
     \end{subfigure}
     \hfill
     \begin{subfigure}[ht]{0.49\textwidth}
         \centering
         \includegraphics[width=0.7\textwidth]{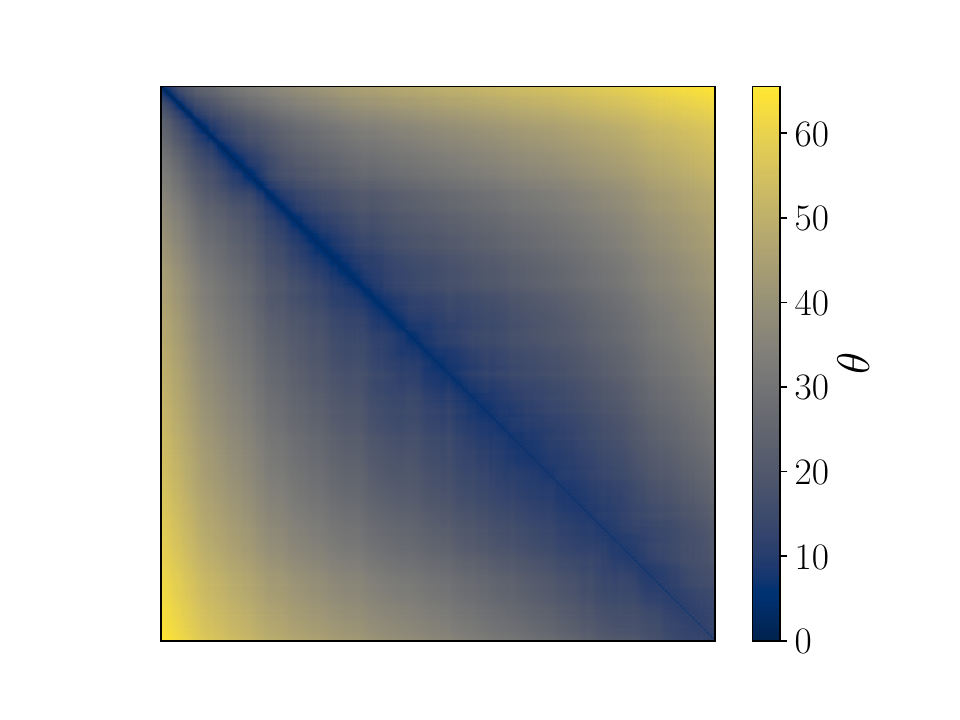}
         \caption{}
         \label{subfig:grad_traj_class}
     \end{subfigure}
     \hfill
        \caption{Angle $\theta$ between the first 1000 gradients obtained during training for the regression model (left) and the classification model (right).}
        \label{fig:grad_trajs}
\end{figure}

\section{EXPERIMENTS WITH THREE-LAYER NEURAL NETWORKS}\label{appendix:3layer}

We provide results showing that the binning phenomenon observed in Section \ref{section:mst} only applies to two-layer neural networks. In other words, three-layer networks did not suffer the under-fitting we observed in Section~\ref{subsection:results_mst}. 

We trained a regression model with two hidden layers consisting of $1000$ and $250$ neurons (totalling $252,250$ parameters) using the square loss. The model was trained for ten random initializations of its weights, in the same manner as detailed in Section \ref{subsection:experiment_setup}. The RMSE for the ten random initializations is depicted by Figure \ref{subfig:3lrmse}.

% \begin{figure}[h]
%     \centering
%     \includegraphics[width=0.40\linewidth]{3lrmse.pdf}
%     \caption{Evaluation RMSE of the model over the ten random initializations of its weights.}
%     \label{fig:3lrmse}
% \end{figure}

\begin{figure}[ht]
     \centering
     \begin{subfigure}[ht]{0.49\textwidth}
         \centering
         \includegraphics[width=0.7\textwidth]{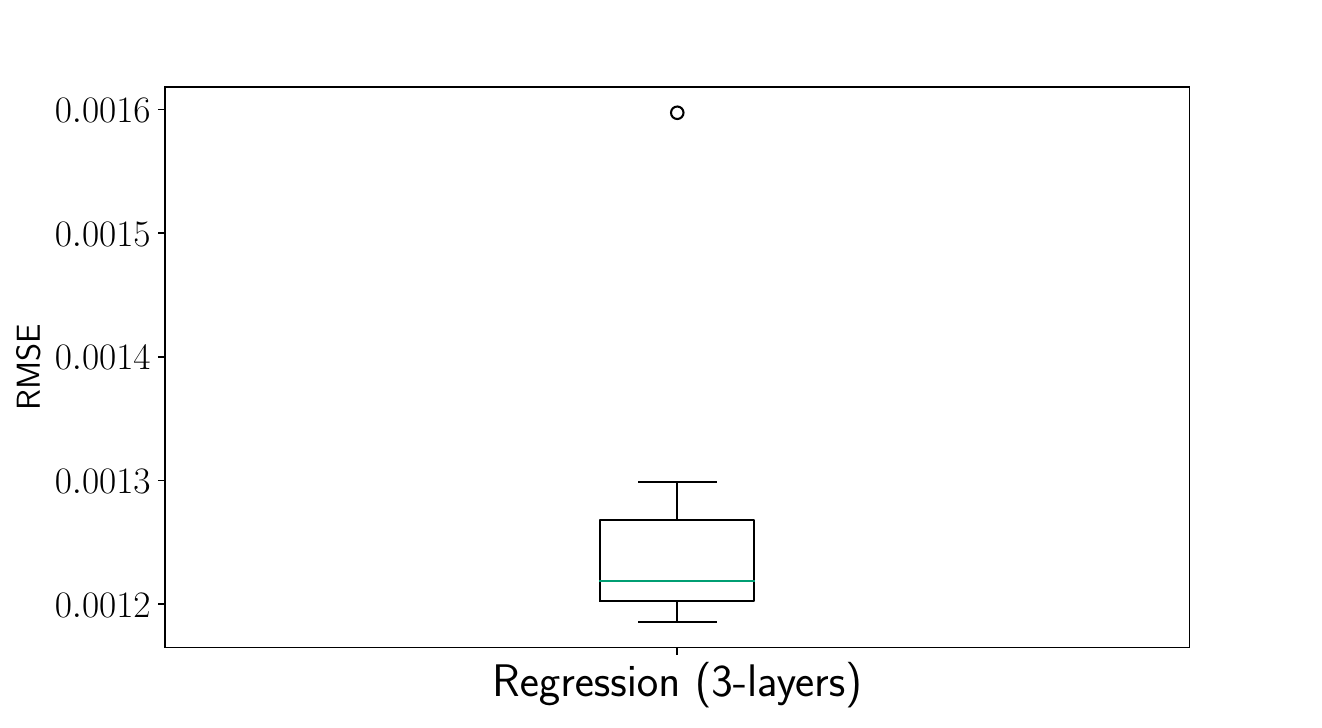}
         \caption{}
         \label{subfig:3lrmse}
     \end{subfigure}
     \hfill
     \begin{subfigure}[ht]{0.49\textwidth}
         \centering
         \includegraphics[width=0.7\textwidth]{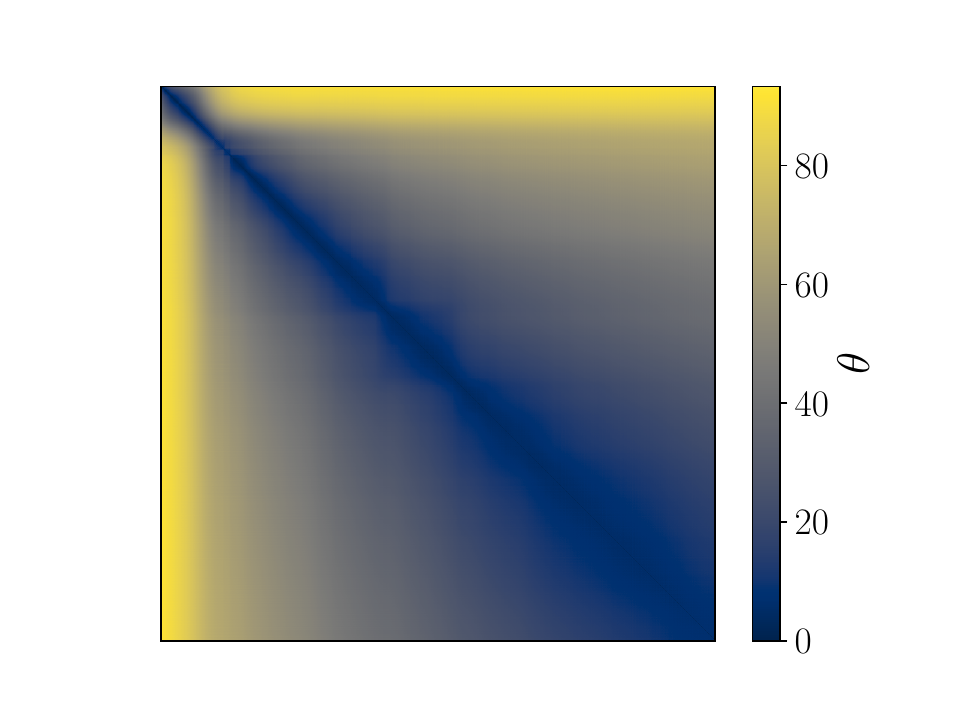}
         \caption{}
         \label{subfig:3ltraj}
     \end{subfigure}
     \hfill
        \caption{RMSE (left) and angle $\theta$ between the first 1000 gradients obtained during training (right) for the three-layer neural network.}
        \label{fig:3lrmse_and_trajs}
\end{figure}

\begin{figure}[!htp]
     \centering
     \begin{subfigure}[h]{0.49\textwidth}
         \centering
         \includegraphics[width=0.7\textwidth]{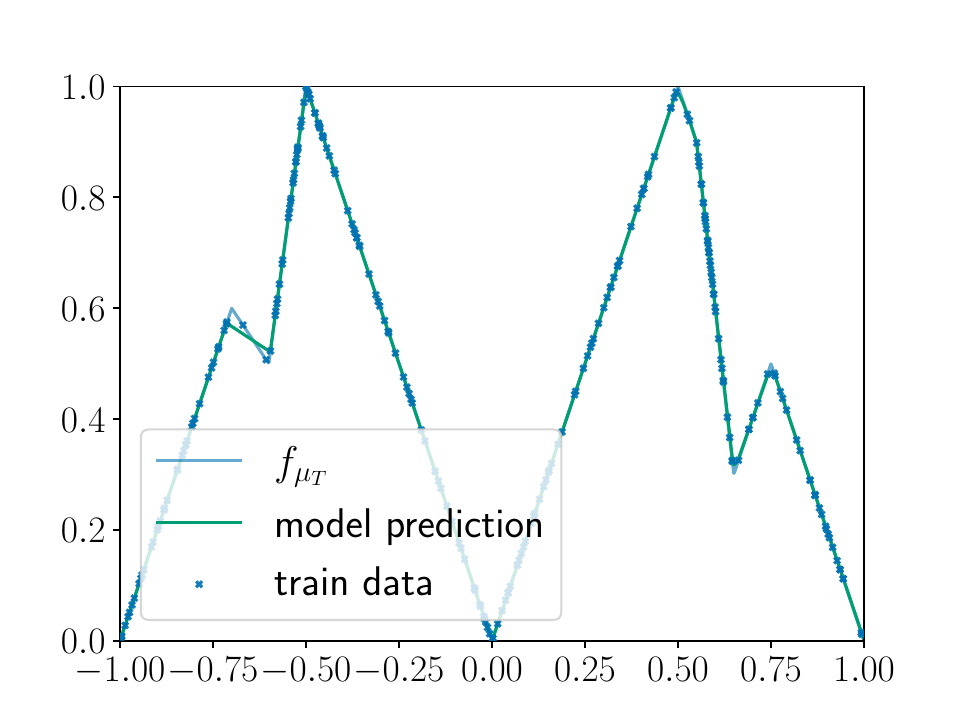}
         \caption{}
         \label{subfig:3lpreds}
     \end{subfigure}
     \hfill
     \begin{subfigure}[h]{0.49\textwidth}
         \centering
         \includegraphics[width=0.7\textwidth]{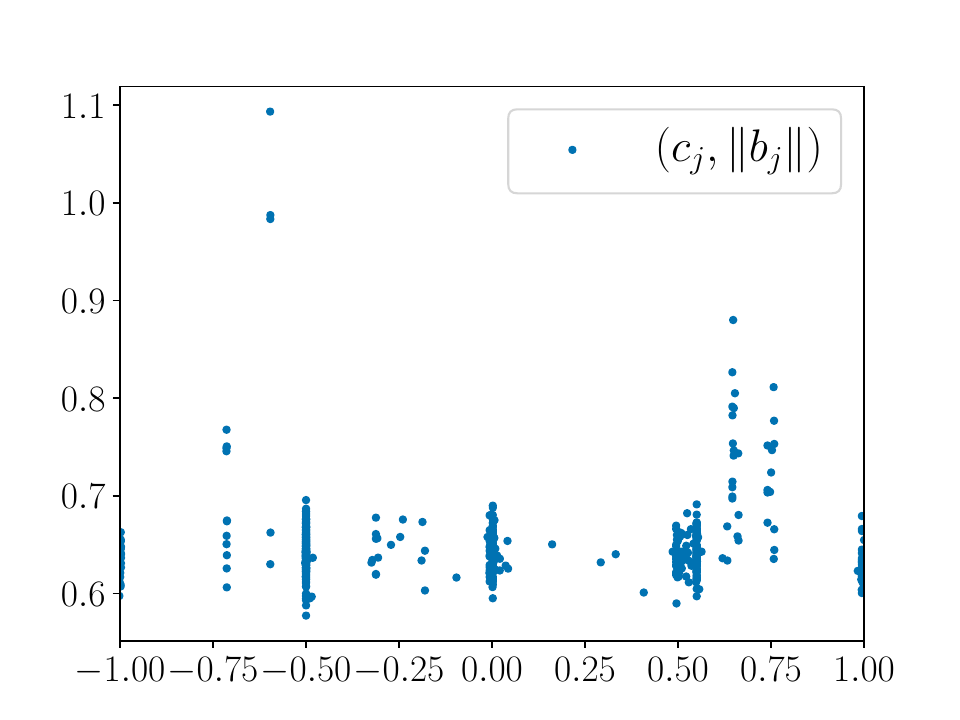}
         \caption{}
         \label{subfig:3lsupp}
     \end{subfigure}
     \hfill
        \caption{Worst performing model predictions (left) and support (right), over the ten random initializations.}
        \label{fig:3lpreds_and_feats}
\end{figure}

Figure \ref{subfig:3ltraj} shows that the gradient descent route during the first 1000 epochs of training was not linear. The predictions of the worst performing model over the 10 runs are depicted in Figure \ref{subfig:3lpreds}. It can be seen that the network did not suffer the under-fitting observed for two-layer networks, detailed in Section \ref{subsection:results_mst}. The support of the model is depicted in Figure \ref{subfig:3lsupp}.

\section{TWO-DIMENSIONAL OPTIMAL SUPPORTS FOR SYNTHETIC DATA}\label{appendix:highdimex}

We generated a regression problem from a random teacher model $\mu_T$ with three neurons, with weights being initialized as by \citet{initializing}. Our train data set consisted of $625$ data points $\left\{\left(x_i, f_{\mu_T}(x_i) \right)\right\}$, where the $x_{i}\in [-1,1]^2\times \{1\}$ are spaced evenly over the $25\times 25$ unit grid. To generate discretized labels we used $k=25$ bins.

We trained two over-parameterized models:

\begin{enumerate}
    \item \textbf{Regression Model:} 500 neurons in the hidden layer with scalar output. Trained using the square loss.
    \item \textbf{Classification Model:} 500 neurons in the hidden layer with vector output of dimension $k=25$. Trained using the cross-entropy loss.
\end{enumerate}

As mentioned in Section \ref{section:highdimex}, each feature $a_j$ is now characterized by the line where it ramps, which we will refer to as the feature's ``critical line''. That is to say, the points $x\in\mathbb{R}^3$ satisfying:
\begin{equation*}
    a_{j,1}x_1 + a_{j,2}x_2 + a_{j,3} = 0,
\end{equation*} 

where $x_3 = 1$. These can be thought of as the equivalent of $c_j$ defined in Section \ref{section:reparam}, but for the two-dimensional case.

The critical lines characterizing the features of the regression and classification models after training are depicted in Figure \ref{fig:3dreg_supp} and \ref{fig:3dclass_sup}, respectively. Features with critical lines which do not cross the unit square only correspond to affine transformations of the resulting prediction, and for this reason can be ignored. Similarly, features killed by the output layer\footnote{That is to say the features $a_j$ such that $\|a_j\|\|b_j\|$ is very small relative to other features.} since their contributions to the model's prediction are irrelevant.

We see that the regression model recovers a sparse support, whilst the classification model's features are more evenly distributed over unit square corresponding to $(x_1,x_2)$. These observations are similar to $R_{reg}$ and $R_{class}$ in the one-dimensional case, suggesting that the difference in implicit bias between regression and classification support we identified in one-dimensional problems may hold in more general situations.

\end{document}